\documentclass{article}

\usepackage[preprint]{neurips_2026}

\usepackage[utf8]{inputenc} 
\usepackage[T1]{fontenc}    
\usepackage{hyperref}       
\usepackage{url}            
\usepackage{booktabs}       
\usepackage{amsfonts}       
\usepackage{nicefrac}       
\usepackage{microtype}      
\usepackage{xcolor}         

\usepackage{amsmath, amsthm, amssymb}
\usepackage{bm}
\usepackage{cite} 
\usepackage{booktabs}
\usepackage{graphicx}
\usepackage{multirow}
\usepackage{tikz}

\newcommand{\Hspace}{\mathcal{H}}
\newcommand{\Xspace}{\mathcal{X}}
\newcommand{\E}{\mathbb{E}}
\newcommand{\phiMap}{\phi}
\newcommand{\KME}{\mu}
\newcommand{\empKME}{\hat{\mu}}

\newtheorem{theorem}{Theorem}
\newtheorem{definition}{Definition}
\newtheorem{lemma}{Lemma}
\newtheorem{proposition}{Proposition}
\usepackage{enumitem}
\usepackage{algorithm}

\usetikzlibrary{patterns}

\newtheorem{assumption}{Assumption}
\newtheorem{corollary}{Corollary}
\newtheorem{remark}{Remark}

\usepackage{enumitem}
\newlist{myitems}{enumerate}{1}
\setlist[myitems, 1]
{label=\arabic{myitemsi}., 
leftmargin=15pt,
rightmargin=10pt
}

\title{Why does Greedy Search produce Optimal Clustering Outcomes? A Fixed-Core Assignment Theory}

\author{%
Kaifeng Zhang \\   
  School of Artificial Intelligence\\
  Nanjing University, China \And
  Kai Ming Ting\\  
   School of Artificial Intelligence\\
  Nanjing University, China 
  \AND
  Sanjay Chawla \\ Qatar Computing Research Institute, HBKU
}

\begin{document}

\maketitle

\begin{abstract}
Many existing clustering methods are designed based on a set-oriented  definition---a cluster is a set of similar points---relying a point-to-point similarity function to find similar points. This works well for compact clusters, but clustering performance can deteriorate badly when cluster shapes are irregular, and densities or sizes vary between clusters. Recent `Cluster-as-Distribution' (CaD) clustering has been shown to discover these generic types of clusters  in practice by treating each cluster as a set of independent and identically distributed points generated from some unknown distribution via a greedy search, achieving a clustering objective equivalent to that of Spectral Clustering, but with better clustering outcomes without eigen-decomposition. However, a theoretical analysis of this phenomenon is still lacking. Our analyses are from two angles. First, we analyze the approximation error between the true  and empirical distribution embeddings. Second, we show that the greedy search employed to achieve the CaD clustering objective can be mapped to a partition matroid---yielding greedy optimality. These yield a near-optimality guarantee for the CaD clustering objective, with regret controlled by the approximation error. This is the first analysis that explains why CaD clustering via greedy search can discover clusters of arbitrary shapes, densities and sizes (where all set-oriented clustering methods have failed to discover) when the estimated cluster embeddings faithfully approximate the underlying cluster distributions.
\end{abstract}


\section{Introduction}

Most existing clustering methods treat \emph{each cluster as a set of similar points} and optimize local similarity structure based on some assumption (e.g., each cluster can be represented with a centroid or Gaussian distribution \citep{aggarwal2015data,macnaughton1964dissimilarity}) relying on a point-to-point similarity/distance function. This set-oriented paradigm is effective for compact, homogeneous clusters, but it is often less reliable when clusters are non-convex, multi-scale, or density-heterogeneous \citep{bhattacharjee2021survey,jain2010data,zhu2016density}. Even more sophisticated clustering such as Spectral Clustering \citep{von2007tutorial,dhillon2004kernel} and Deep Clustering \citep{wu2022deep} have their fundamental limitations \citep{lu2024survey,nadler2006fundamental,ros2024deep,ting2026achieve,zhou2024comprehensive}.  This is because the assumed cluster structure (in either the input space or the learned representation space) often does not match the distributions in which the clusters are generated.  In addition, all these algorithms have time complexity worse than quadratic, except $k$-means clustering.

A recent `Cluster-as-Distribution' (CaD) Clustering \citep{ting2026achieve,zhang2025kernel,zhu2023kernel,ting2026achieve} addresses this mismatch by treating each cluster as a set of independent and identically distributed (i.i.d.) samples generated from an underlying distribution. In this setting, the task can be stated as follows:
\textit{Given unlabeled points in $\mathbb{R}^d$, discover $k$ clusters as $k$ sets of i.i.d. samples generated from $k$ distinct unknown distributions with potentially arbitrary shapes, sizes, and densities.}

Example methods of CaD Clustering are Kernel-Bounded Clustering (KBC) \citep{zhang2025kernel,zhang2025kernel2}, Isolation Distributional Kernel Clustering (IDKC) \citep{zhu2023kernel}, and Point-Set Kernel Clustering (psKC) \citep{ting2022point}, and they have produced better clustering outcomes than kernel $k$-means, Spectral Clustering and deep clustering. The unique difference of CaD Clustering is utilizing the distributional information in a given dataset to perform clustering via a distributional kernel, without parametric estimation (as in Gaussian Mixture Modeling \citep{gormley2023model,reynolds2009gaussian}) or density estimation \citep{rinaldo2012stability} (as in density-based clustering \citep{bhattacharjee2021survey,ester1996density}). \textit{The shared clustering procedure of CaD Clustering is simple: construct a distribution proxy $\mathcal{P}_G$ of core cluster $G \subset C$, for each to-be-discovered cluster $C$, then assign each point $x$ in the given dataset to the most similar proxy based on a distributional kernel $K(\delta(x), \mathcal{P}_G)$}, without using a sophisticated optimization such as deep learning, Expectation-Maximization (EM) optimization or  eigen-decomposition, where $\delta(x)$ is the Dirac measure which converts a point into a distribution, and $K$ measures the similarity between two distributions. The CaD clustering objective is to maximize the total  similarity between points and their assigned proxies, as measured by $K$, via a greedy search.


Despite this strong clustering performance, a fundamental question remains unanswered:
Why does a greedy search method achieve better clustering outcomes than deep learning, EM optimization, or eigen-decomposition based methods (as reported in the literature \citep{ting2022point,zhang2025kernel,zhang2025kernel2,ting2026achieve,zhu2023kernel})?
This paper provides an explanation towards answering this question.
In particular, we answer two central questions: (i) How accurately does a core cluster $G \subset C$ represent a cluster distribution $\mathcal{P}_C$? (ii) Once the core clusters (having approximation errors) are fixed, does the one-pass assignment, via a greedy search based on  $K(\delta(x), \mathcal{P}_G)$, achieve the clustering  objective optimally?

This paper answers these two questions within a single framework. Our contributions are:
\begin{myitems}
    \item \textbf{Decomposition of distribution embedding approximation error}
    into three components: truncation, estimation, and core selection. The core selection term depends on how representative the selected proxies are and it is the main factor that affects the approximation quality.
    \item \textbf{Fixed-core assignment optimality} is
    formulated via an one-pass assignment with fixed-core embeddings as a
    maximum-weight basis selection problem on a partition matroid \citep{edmonds1971matroids}, showing that
    the greedy rule is optimal for the CaD clustering objective.
\item \textbf{Distribution regret under embedding approximation error} links
    the  error to the  CaD clustering objective via the partition matroid---providing a margin-based label
    recovery corollary.

    \item \textbf{Diagnostic experiments} validate the error
    decomposition on synthetic benchmarks, showing that all the bounds are valid, while the core bias bound is
    conservative because it is model-agnostic. We also compare CaD clustering with kernel $k$-means, where both share many similarities but have important differences.
\end{myitems}

These results provide the first theoretical understanding on why CaD clustering via a greedy search can be effective from a statistical and combinatorial perspective. Notations are summarized in Table~\ref{tab:notation}.

\section{Related Work and Preliminaries}



\subsection{`Cluster-as-Distribution' Clustering}


`Cluster-as-Distribution' (CaD) clustering methods replace the typical set view with a distribution view, where each cluster is defined to be a set of i.i.d. samples generated from an underlying distribution. This can be represented via a distributional kernel based on Kernel Mean Embedding (KME) \citep{muandet2017kernel,smola2007,sriperumbudur2010hilbert}, without Gaussian distribution assumption or density estimation. 


Representative methods
KBC \citep{zhang2025kernel,zhang2025kernel2}, psKC \citep{ting2022point}, IDKC \citep{zhu2023kernel} can be understood to operate under a two-stage approach (see Appendix \ref{appendix:algorithm-details} for the detailed algorithmic descriptions): (1) Representing the distribution of each cluster in input space, via KME, as a point in a Reproducing Kernel Hilbert Space (RKHS). 
(2) Performing clustering via one-pass assignment where each individual data point is assigned to the most similar cluster, as measured by a distributional kernel and computed in RKHS.

Let $k: \Xspace \times \Xspace \to \mathbb{R}$ be the Gaussian kernel $k(x, y) = \exp(-\|x-y\|^2/2\sigma^2)$, and let $\phiMap(x) = k(x, \cdot) \in \Hspace$ be its feature map in RKHS. Denote by $\mathcal{P}$ the underlying continuous distribution of a target cluster and by $G = \{x_1, \dots, x_m\}$ the core cluster (a finite subset of the support of $\mathcal{P}$) found by a specific method (e.g., connected component). KBC optimizes the total self-similarity objective, i.e., $\max_{\mathcal{C}} \sum_{C\in\mathcal{C}}\sum_{x\in C} K(\delta(x),\mathcal{P}_C)$, where $K$ is the distributional kernel induced by point kernel $k$, $\mathcal{C}$ is the set of all possible clustering results and $\mathcal{P}_C$ is the distribution of cluster $C$.

\begin{definition}
The mean embedding of distribution $\mathcal{P}$ is defined as $\KME_{\mathcal{P}} = \int_{\Xspace} \phiMap(x) d\mathcal{P}(x).$
\end{definition}

\begin{definition}
The empirical mean embedding based on core cluster $G$, which is a set of i.i.d. samples from $\mathcal{P}$, is defined as $\empKME_G = \frac{1}{|G|} \sum_{x \in G} \phiMap(x).$
\end{definition}


The central quantity that we are about to analyze in Section \ref{section: KME error} is the  discrepancy between the two embeddings in RKHS: $\Delta = \|\KME_\mathcal{P} - \empKME_G \|_{\Hspace}$. Our analysis is also related to the broader literature on Kernel Mean
Embeddings \citep{muandet2017kernel,smola2007,sriperumbudur2010hilbert} and Maximum Mean Discrepancy \citep{gretton2006kernel}. The standard KME theory
provides approximation tools for distributional
representations in RKHS.
Our work differs in that it decomposes the  approximation
error $\Delta$ of the core cluster embedding into truncation, estimation, and core selection terms, and then links
this decomposition to one-pass assignment via a greedy search in achieving the optimality of the CaD clustering objective.

\subsection{Matroid Theory}
A matroid \citep{edmonds1971matroids} abstracts the notion of linear independence in vector spaces to general sets. Matroids identify an important class of feasible-set systems in which the locally greedy choice is guaranteed to produce a globally optimal solution for linear objectives \citep{edmonds1971matroids,faigle2009general}.

\begin{definition}[Matroid \citep{edmonds1971matroids}]
A matroid $M$ is a pair $(S, \mathcal{I})$, where $S$ is a finite ground set and $\mathcal{I} \subseteq 2^S$ is a collection of independent sets satisfying:
(1) \textbf{Non-emptiness:} $\emptyset \in \mathcal{I}$. (2) \textbf{Hereditary Property:} If $A \in \mathcal{I}$ and $B \subseteq A$, then $B \in \mathcal{I}$.
(3) \textbf{Exchange Property:} If $A, B \in \mathcal{I}$ and $|A| < |B|$, there exists an element $e \in B \setminus A$ such that $A \cup \{e\} \in \mathcal{I}$.
\end{definition}

\begin{theorem}[Rado-Edmonds Theorem \citep{edmonds1971matroids,rado1957note}]
\label{theorem: Rado-Edmonds Theorem}
Let $M=(S, \mathcal{I})$ be a matroid. For any non-negative weight function $w: S \to \mathbb{R}^+$, the greedy algorithm (which iteratively picks the heaviest element that maintains independence) is guaranteed to find a set $A \in \mathcal{I}$ that maximizes $\sum_{e \in A} w(e)$.
\end{theorem}

\section{Error Analysis of KME approximation}
\label{section: KME error}
We derive an upper bound for $\Delta$ by decomposing it into truncation, estimation, and core selection components. All proofs are provided in Appendix \ref{appendix: proofs}. 

\begin{figure}[ht]
    \centering
    \begin{minipage}[t]{0.3\textwidth}
        \centering
        \includegraphics[width=\linewidth]{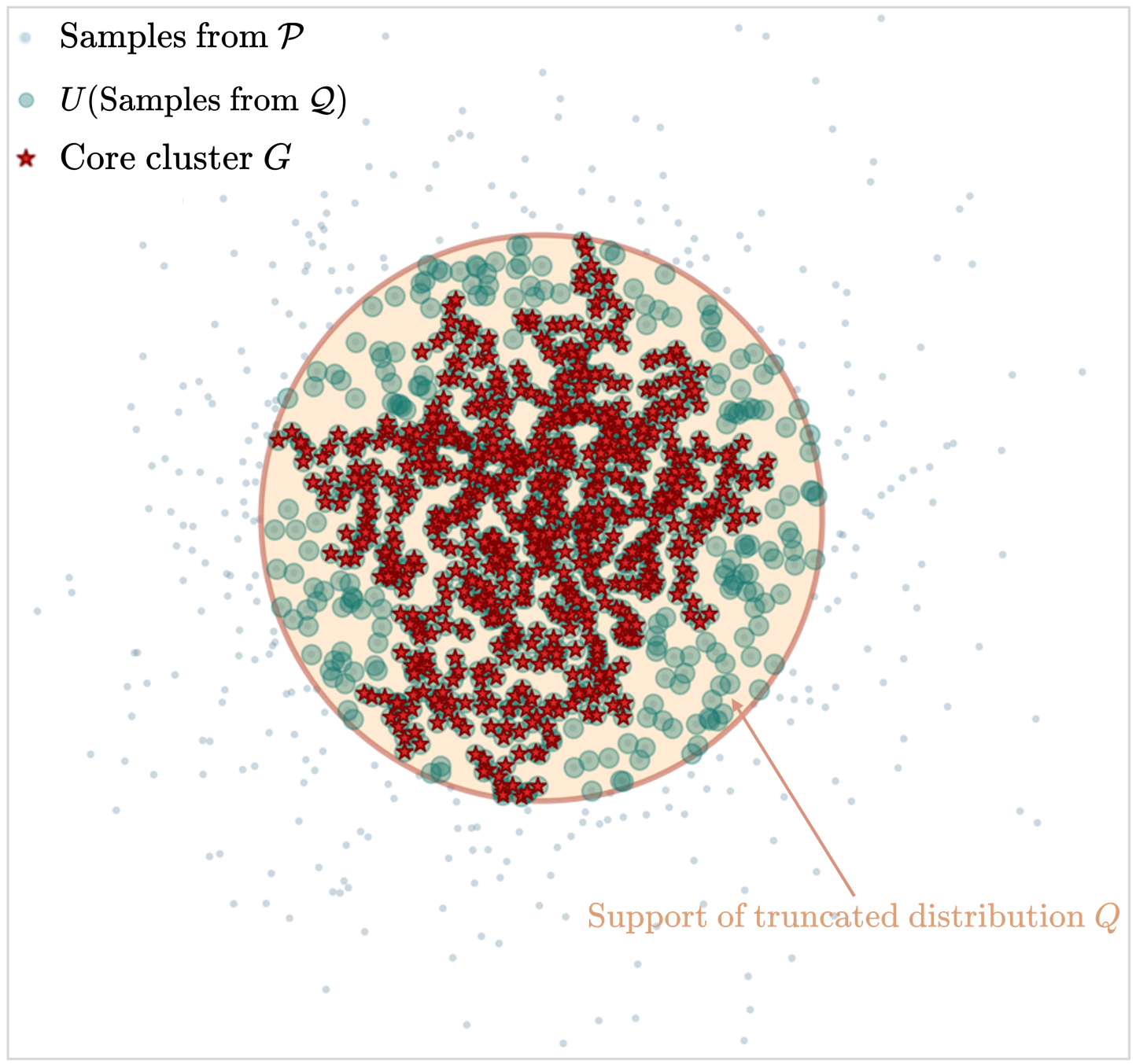}
    \end{minipage}\hspace{0.003\linewidth}
    \begin{minipage}[t]{0.42\textwidth}
        \centering
        \includegraphics[width=\linewidth]{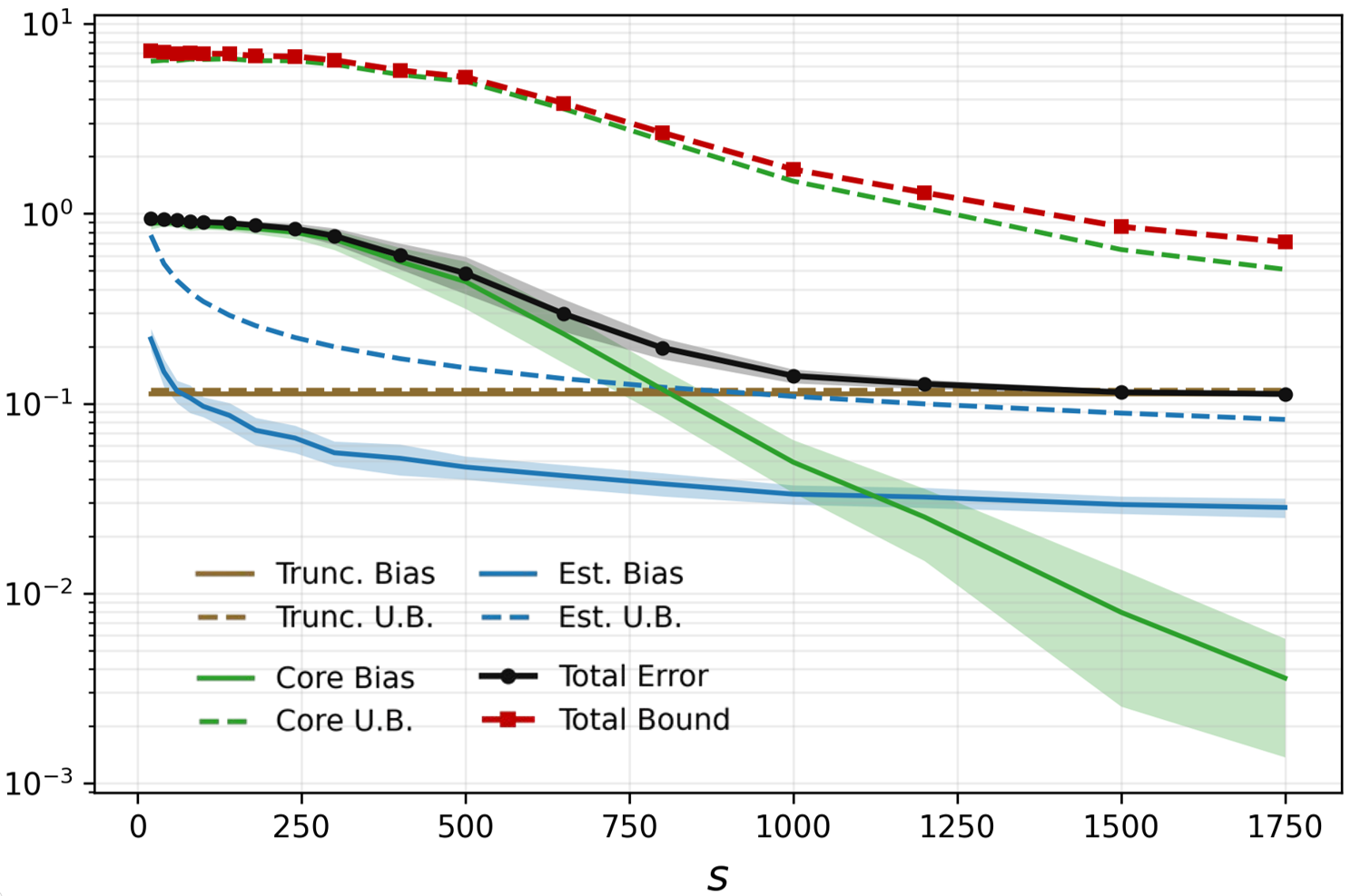}
    \end{minipage}
    \caption{\textit{Left}: visualization on a 2D Gaussian distribution $\mathcal{P}=\mathcal{N}(\mathbf{0}, I)$, showing dense region $\mathcal{Q}$, sampled points $U$ from $\mathcal{Q}$, and core cluster $G$ obtained by KBC (largest connected component under kernel-threshold graph). \textit{Right}: decomposition view of Theorem \ref{theorem: total error} versus $s$ (the sample size from $\mathcal{Q}$, i.e., $|U|$), including the empirical and upper bound curves for truncation bias (Lemma \ref{lemma: tbias}), estimation bias (Lemma \ref{lemma ebias}), core bias (Lemma \ref{lemma: cbias}), together with empirical total error and the Theorem \ref{theorem: total error} bound. Each point is averaged over $30$ repeated samplings, and the band is mean$\pm$1 std.}
    \label{fig:error demo}
\end{figure}


\subsection{Error Decomposition}

The total error $\Delta = \|\KME_\mathcal{P} - \empKME_G\|_{\Hspace}$ is decomposed into three components by introducing the truncated distribution $\mathcal{Q}=\mathcal{P}|\mathcal{S}_{\lambda}$ ($\mathcal{S}_{\lambda} = \{x \in \text{supp}(\mathcal{P}) \mid f_\mathcal{P}(x) \ge \lambda\}$), i.e., the distribution supported on high density region of $\mathcal{P}$ . $\empKME_{U}$ is the empirical estimated KME of the set $U$ of finite points  sampled from $\mathcal{Q}$, which are the candidates of core cluster $G$:
\begin{equation*}
    \|\KME_\mathcal{P} - \empKME_G\|_{\Hspace} \le \underbrace{\|\KME_\mathcal{P} - \KME_\mathcal{Q}\|_{\Hspace}}_{\text{Truncation Bias}} + \underbrace{\|\KME_\mathcal{Q} - \empKME_U\|_{\Hspace}}_{\text{Estimation Bias}} + \underbrace{\|\empKME_U - \empKME_G\|_{\Hspace}}_{\text{Core Bias}}.
\end{equation*}

A demonstration of this decomposition is provided in Figure \ref{fig:error demo}.

\subsection{Truncation Bias}

We separate $\mathcal{P}$ into truncated distribution $\mathcal{Q}$ and tail distribution $\mathcal{T}$. Let $\eta = \mathcal{P}(\mathcal{X} \setminus \mathcal{S}_{\lambda})$ be the probability mass of the tail. We show the relation between the truncation bias (difference in kernel space) and $\mu_\mathcal{T} - \mu_\mathcal{Q}$ (difference in original space) in Lemma \ref{lemma: tbias}.
\begin{lemma}
\label{lemma: tbias}
    When using Gaussian kernel $k_\sigma(x, y)$, the truncation bias satisfies $        \|\mu_\mathcal{P} - \mu_\mathcal{Q}\|_{\mathcal{H}}
        \leq \|\mathbb{E}[\mathcal{Q}]-\mathbb{E}[\mathcal{T}]\|\cdot \frac{\eta}{\sigma} + O\!\left(\frac{1}{\sigma^2}\right),$
    where $\eta = \mathcal{P}(\mathcal{X} \setminus \mathcal{S}_{\lambda})$ is the probability mass of the tail.
\end{lemma}

Note that the $\|\mu_\mathcal{P} - \mu_\mathcal{Q}\|_{\mathcal{H}} \approx \|\mathbb{E}[\mathcal{Q}]-\mathbb{E}[\mathcal{T}]\|\cdot \frac{\eta}{\sigma}$ only holds when $\sigma$ is very large. In practice, when using Gaussian kernel with a small $\sigma$, there is a significant difference between the truncation bias and the first term of the right side in the equality above. Appendix \ref{appendix:tbias-empirical} confirms this trend: for $\mathcal{P}=\mathcal N(0,I)$, the gap between empirical truncation bias and its leading term is larger at small $\sigma$ and shrinks as $\sigma$ increases. The $O\!\left(\frac{1}{\sigma^2}\right)$ term can be either negative or positive due to hyperparameter setting, as shown in Figure \ref{fig: truncation_sigma_lambda_appendix}.

In Lemma \ref{lemma: tbias}, a smaller tail mass excluded will lead to a more accurate approximation of $\mathcal{Q}$ to $\mathcal{P}$. For regular symmetric distributions such as a 2D Gaussian, this truncation effect can be quite small. For asymmetric or complex-shaped distributions, the truncated region may deviate more substantially due to skewness, heavy tails, or irregular support, and the resulting truncation bias can therefore be larger.

\subsection{ Estimation Bias}

For the truncation distribution $\mathcal{Q}$ and a set of points $U$ sampled from $\mathcal{Q}$ with size $s$, we prove the $O(1/\sqrt{s})$ convergence of the estimation bias in Lemma \ref{lemma ebias} .

\begin{lemma}
[Convergence of the Empirical Embedding of Dense Region]
\label{lemma ebias}
For a sample of $s$ i.i.d. points $U$ from the truncated distribution $\mathcal{Q}$ and a Gaussian kernel with $k(x,x) \le 1$, it holds with probability at least $1-\delta$ that $\|\KME_\mathcal{Q} - \empKME_U\|_{\Hspace} \le {\left(1 + \sqrt{2 \ln(1/\delta)}\right)}/{\sqrt{s}}$.
\end{lemma}

Appendix \ref{appendix:ebias-empirical} supports Lemma \ref{lemma ebias}: with fixed $\delta=0.01$, the empirical estimation bias decreases with $s$ by following a $O(1/\sqrt{s})$. This estimation bias is only related to subsample size $s$, which has a linear ratio of the dataset size $n$. So we can also treat the estimation bias as $O(1/\sqrt{n})$.

\subsection{Core Bias}

The core bias term depends on the quality of the core cluster, a subset of the set of points sampled from $\mathcal{Q}$ with size $s$. Lemma \ref{lemma: cbias} provides a \textbf{method-agnostic} upper bound under a coupling condition. The quality of the core cluster is directly reflected in parameter $r$ \footnote{In the experiments, we use nearest neighbor projection as mapping $\pi$.} and determined by a specific core cluster construction method. Appendix \ref{appendix:cbias-empirical} provides a direct empirical check of this inequality.

\begin{lemma}[A method-agnostic upper bound for core bias]
\label{lemma: cbias}
Let $U=\{q_i\}_{i=1}^s$ be sampled points from $\mathcal{Q}$ and let $G=\{g_j\}_{j=1}^m$ be core points. Assume there exists a mapping $\pi:\{1,\dots,s\}\to\{1,\dots,m\}$ such that $    \|q_i-g_{\pi(i)}\|\le r, \quad \forall i\in\{1,\dots,s\}.$ Define occupancy proportions $p_j=\frac{1}{s}\left|\{i:\pi(i)=j\}\right|,\quad j=1,\dots,m,$ and imbalance term $    B=\sum_{j=1}^m\left|p_j-\frac{1}{m}\right|.$ Then for Gaussian kernel $k_\sigma$, it holds that $    \|\empKME_U-\empKME_G\|_{\mathcal{H}}
    \le \frac{r}{\sigma} + B.$

\end{lemma}

\subsection{Total Error}

With the analysis above, we reach Theorem \ref{theorem: total error}. Figure \ref{fig:error demo} (right) provides an empirical reflection of Theorem \ref{theorem: total error} under $\sigma=0.35$, $\tau=0.95$, $\delta=0.05$, and $\lambda=0.05$. 

\begin{theorem}
\label{theorem: total error}
When using Gaussian kernel $k_\sigma(x, y)$ in CaD clustering, with probability at least $1-\delta$, it holds that
\begin{equation*}
    \|\KME_\mathcal{P} - \empKME_G\|_{\Hspace}\leq  \underbrace{\|\mathbb{E}[\mathcal{Q}]-\mathbb{E}[\mathcal{T}]\|\cdot \frac{\eta}{\sigma} + O\!\left(\frac{1}{\sigma^2}\right)}_{\text{distribution-dependent}} +    \underbrace{\frac{1 + \sqrt{2 \ln(1/\delta)}}{\sqrt{s}}}_{\text{sample-dependent}}+\underbrace{\left(\frac{r}{\sigma}+B\right)}_{\text{method-dependent}}.
\end{equation*}
\end{theorem}

The bandwidth $\sigma$ plays a central role in the total error bound because it controls the distance sensitivity of the Gaussian kernel $k_\sigma(x,y)=\exp(-\|x-y\|^2/(2\sigma^2))$. When $\sigma$ is small, the kernel value decays rapidly as the distance increases, making the embedding more sensitive to tail discrepancy and mismatch between sampled points and selected cores. When $\sigma$ is larger, the kernel varies more smoothly with distance, so these discrepancies induce smaller perturbations in RKHS, which is reflected in the $1/\sigma$ dependence of the truncation bias and core bias. Hence, $\sigma$ determines a tradeoff between preserving fine local structure and improving approximation robustness.

\noindent The error analysis in Theorem \ref{theorem: total error} establishes that the reliability of CaD clustering depends on three factors: distribution truncation quality, sample size ($s$) of dense region, and core cluster selection quality. 

\begin{myitems}
    \item \textbf{Distribution-dependent Truncation Bias.} The analysis shows that the truncation bias depends on how well the dense region $\mathcal{Q}$ approximates the underlying distribution $\mathcal{P}$. This term is distribution-dependent: for a symmetric distribution like a 2d Gaussian distribution in $\mathbb{R}^2$, this term can be very small. For asymmetric or complex-shaped distributions, the truncation bias can be larger due to heavier tails or irregular support.
    \item \textbf{Sample-dependent Estimation Bias.} The parameter $s$ represents the number of sampled points drawn from the dense region $\mathcal{Q}$ to estimate $\mu_\mathcal{Q}$. In practice, the $O(1/\sqrt{s})$ convergence rate means that increasing $s$ reduces statistical estimation noise. 
    \item \textbf{Method-dependent Core Bias.} The core bias contribution quantifies how well core cluster represents sampled points from the dense region $\mathcal{Q}$. By Lemma \ref{lemma: cbias}, this term is upper bounded by $r/\sigma + B$, where $r$ is the geometric approximation radius and $B$ is the sample-to-core imbalance, which are both determined by the core cluster construction method.
\end{myitems}

For a practitioner, this provides a mathematical guarantee that a one-pass assignment is reliable provided core clusters are representative of the underlying distribution and the number of points in dataset is large enough.

\textbf{Conservativeness of the Total Error Bound Due to Core Bias.} The looseness of the total bound in Figure \ref{fig:error demo} is mainly caused by the core bias term, whose upper bound decreases more slowly than the empirical estimate. This is expected because the analysis uses a method-agnostic upper bound, which is intended to apply across different core construction procedures and is therefore conservative for any particular method. Nevertheless, Figure \ref{fig: core_bias_tau_appendix} in Appendix shows that the bound can become much tighter under certain hyperparameter settings.

\textbf{Remark on estimation and core bias.} Although the estimation bias decreases with sample size $s$, the overall approximation quality may still be bad when the initial core is not representative, since a small sample size could cause the quality of core clusters to deteriorate. For KBC, a specific CaD method, Appendix \ref{app:kbc cbias} gives a complementary interpretation of the core selection term through density level sets. In the Gaussian kernel case, the KBC threshold $\tau$ induces a graph radius $\varepsilon_\tau=\sigma\sqrt{2\log(1/\tau)}$. The KBC core bias can be controlled by the boundary band mass plus a finite sample deviation. This result is conditional on the level set connectivity assumption, so it should be viewed as a KBC-specific certificate rather than an unconditional replacement for Lemma \ref{lemma: cbias}.



\section{Modeling of the Assignment via Distributional Kernel as Matroid }
The `cluster-as-distribution' assumption has enabled the one-pass assignment via greedy search---such that proving it via matroid is simple---to produce the optimal clustering outcome. Once each core is represented as a distributional proxy, each point-to-cluster weight is independent of other assignments, yielding an additive objective that decomposes over points and fits a partition matroid.
\textit{The partition matroid 
simply formalizes the exact optimality of the fixed-core one-pass assignment and separates it from the core-cluster construction problem.}

\subsection{Matroid Axioms for CaD clustering}
\textbf{Ground Set.} Let $D = \{x_1, x_2, \dots, x_n\}$ be the dataset and $\mathcal{K} = \{1, 2, \dots, k\}$ be the set of cluster indices. We define the ground set $\mathcal{S}_{\mathrm{g}}$ as the set of all possible point-to-cluster assignments:
$\mathcal{S}_{\mathrm{g}} = D \times \mathcal{K} = \{ (x_i, j) \mid 1 \le i \le n, 1 \le j \le k \}$.

\textbf{Partition Matroid.} To ensure each point $x_i$ is assigned to at most one cluster, we partition $\mathcal{S}_{\mathrm{g}}$ into $n$ disjoint blocks $S_1, S_2, \dots, S_n$, where:
$S_i = \{ (x_i, 1), (x_i, 2), \dots, (x_i, k) \}$.
We define the family of independent sets $\mathcal{I}$ as $\mathcal{I} = \{ A \subseteq \mathcal{S}_{\mathrm{g}} \mid |A \cap S_i| \le 1, \forall i \in \{1, \dots, n\} \}$.

\begin{proposition}
\label{prop:matroid_structure}
The pair $(\mathcal{S}_{\mathrm{g}}, \mathcal{I})$ as defined for the CaD clustering is a matroid.
\end{proposition}
The proof is provided in Appendix \ref{appendix:proof-matroid-structure}. The rationale behind $|A\cap S_i|\leq 1$ is given in Appendix \ref{appendix: mr}. 




\subsection{Optimality of the One-Pass Assignment}
\label{sect: opt assignment}
To start the analysis on the assignment, we first define a weight function to link matroid and the objective of CaD clustering.   

\textbf{Weight Definition.}
Given the pre-identified cluster cores $G_1, \dots, G_k$ with empirical embeddings $\hat\mu_{G_j}$, the weight of an assignment $e = (x_i, j)$ is defined as the RKHS inner product score:
$w(x_i, j) = \langle \phiMap(x_i), \hat\mu_{G_j} \rangle_{\Hspace}$.
The objective of the CaD clustering during the assignment phase is to maximize $W(A) = \sum_{e \in A} w(e)$ for a set $A$ that covers all points, i.e., $|A|=n$.

The CaD clustering algorithm performs the following greedy step for each point $x_i \in D$:
$j^* = \arg\max_{j \in \mathcal{K}} \langle \phiMap(x_i), \hat\mu_{G_j} \rangle_{\Hspace}$. The one-pass assignment can be viewed as a greedy procedure on the partition matroid $(S_{\mathrm{g}}, \mathcal{I})$. Since the objective is additive over assignment elements, maximizing the total weight reduces to selecting, for each point $x_i$, the assignment $(x_i,j)$ with the largest weight $w(x_i,j)$. Therefore, the one-pass rule coincides with a greedy selection on the partition matroid. Finally, given fixed core clusters, the optimality is guaranteed by Theorem \ref{theorem: Rado-Edmonds Theorem}, as shown in Theorem \ref{theorem: fixed_core_opt}.

\begin{theorem}[Optimality for fixed-core assignment
]
\label{theorem: fixed_core_opt}
Given some fixed core clusters $G_1,\dots,G_k$, and define additive weights $w(x_i,j)=\langle \phiMap(x_i),\hat\mu_{G_j}\rangle_{\Hspace}$ on the partition matroid $(\mathcal{S}_{\mathrm{g}},\mathcal{I})$. The one-pass assignment that selects
$j_i^*\in \arg\max_{j\in\mathcal K} w(x_i,j)$ for all $i$, produces a maximum-weight basis and is optimal for the assignment objective $\max_{A\in\mathcal I} \sum_{e\in A} w(e)$.
\end{theorem}
The proof is provided in Appendix \ref{appendix:proof-fixed-core-opt}. 

\textbf{The above optimality statement is conditional on fixed core clusters.
It does not claim global optimality for the full clustering procedure, since
the quality of the final result depends critically on how well the
core clusters approximate the cluster distributions.}


\section{From KME Approximation Error to CaD Clustering Objective}
Theorem \ref{theorem: fixed_core_opt} establishes the optimality of the objective induced by fixed-core embeddings. We now connect this result to the CaD clustering objective defined by the true distributions.

For cluster $j$, let $\mu_j=\KME_{\mathcal{P}_j}$ be the distribution embedding and $\hat\mu_j=\empKME_{G_j}$ be its empirical proxy. Define $w^\star(x_i,j)=\langle \phiMap(x_i), \mu_j \rangle_{\Hspace}$ and $\hat w(x_i,j)=\langle \phiMap(x_i), \hat\mu_j \rangle_{\Hspace}$. For any feasible assignment set $A\in\mathcal I$ with $|A|=n$, define $W^\star(A)=\sum_{(x_i,j)\in A} w^\star(x_i,j)$ and $\hat W(A)=\sum_{(x_i,j)\in A} \hat w(x_i,j)$.

\begin{proposition}[Near-optimality for the assignment objective]
\label{prop:near_opt_population_obj}
Assume $\|\phiMap(x)\|_{\Hspace}\le 1$ for all $x$, and define $\varepsilon=\max_{j\in\mathcal K}\|\mu_j-\hat\mu_j\|_{\Hspace}$. Let $A^\star\in\arg\max_{A\in\mathcal I,\,|A|=n} W^\star(A)$ and $\hat A\in\arg\max_{A\in\mathcal I,\,|A|=n} \hat W(A)$.
Then $W^\star(A^\star)-W^\star(\hat A)\le 2n\varepsilon.$

\end{proposition}
The proof is provided in Appendix \ref{appendix:proof-near-opt-population}. 

The following links the KME error in Theorem \ref{theorem: total error} with Proposition \ref{prop:near_opt_population_obj} and show how the KME error upper bounds the objective in Proposition \ref{cor:regret_decomposed_terms} with the proof provided in Appendix \ref{appendix:proof-regret-decomposed}.
\begin{proposition}[Relation with KME error]
\label{cor:regret_decomposed_terms}
Suppose each cluster $j\in\mathcal K$ satisfies $\|\mu_j-\hat\mu_j\|_{\Hspace}\le b_j$,
where $b_j :=
\|\mathbb E[\mathcal{Q}_j]-\mathbb E[\mathcal{T}_j]\|\cdot\frac{\eta_j}{\sigma}
+ O\!\left(\frac{1}{\sigma^2}\right)
+ \frac{1+\sqrt{2\ln(1/\delta)}}{\sqrt{s_j}}
+ \frac{r_j}{\sigma}
+ B_j.$
Then the one-pass assignment $\hat A$ satisfies $W^\star(A^\star)-W^\star(\hat A)
\le 2n\max_{j\in\mathcal K} b_j.$
\end{proposition}

\section{Label Recovery}

In Theorem \ref{thm:exact_label_recovery}, we formalize the `optimal clustering'---all  labels assigned by the distribution kernel equal to the ground truth: label recovery is guaranteed when the embedding error is smaller than half of the distribution score margin.

\begin{theorem}[Label recovery under error]
\label{thm:exact_label_recovery}
Let $y_i^\star \in \arg\max_{j\in\mathcal K} w^\star(x_i,j)$ and $w^\star(x_i,j)=\langle\phiMap(x_i),\mu_j\rangle_{\Hspace}$,
and assume the distribution top-two margin is strictly positive, i.e., $\gamma_\star := \min_{1\le i\le n}\Big(w^\star(x_i,y_i^\star)-\max_{j\neq y_i^\star}w^\star(x_i,j)\Big) > 0.$

If $\varepsilon:=\max_{j\in\mathcal K}\|\mu_j-\hat\mu_j\|_{\Hspace} < \frac{\gamma_\star}{2}$,
then the one-pass proxy assignment
$\hat y_i\in\arg\max_{j\in\mathcal K}\hat w(x_i,j)$ with $\hat w(x_i,j)=\langle\phiMap(x_i),\hat\mu_j\rangle_{\Hspace}$,
recovers all labels exactly, i.e., $\hat y_i=y_i^\star$ for all $i$.
\end{theorem}


\begin{proposition}[Recovery condition via decomposition terms]
\label{prop:exact_recovery_via_decomposition}
Under the assumptions of Theorem \ref{thm:exact_label_recovery}, if
$\max_{j\in\mathcal K} b_j < \frac{\gamma_\star}{2}$,
where $b_j$ is defined in Proposition \ref{cor:regret_decomposed_terms}, then the label recovery holds for all points.
\end{proposition}

\noindent This shows that increasing dense region sample size $s_j$, improving core quality (smaller $r_j,B_j$); and reducing the truncation error tightens $b_j$, making the recovery condition easier to satisfy.

\section{Experiments}


\subsection{Empirical Validation of Error Decomposition}
\label{sec:exp_decomp}

We aim to verify empirically that the three error components identified in
Theorem~\ref{theorem: total error} (truncation, estimation, and core
bias) decrease as the sample size $n$ increases, and that the
corresponding theoretical upper bounds remain valid and capture the observed
decay trend. The core bias bound is expected to be conservative because Lemma~\ref{lemma: cbias} is model-agnostic and is
designed to apply across different core construction procedures. We use KBC \citep{zhang2025kernel2} as the representative CaD clustering method~\footnote{Additional GDKC \citep{zhu2023kernel} evidence are deferred to Appendix \ref{app:gdkc}.} in our experiments, since it has the simplest mechanism among the three methods considered and is the easiest to interpret in light of our analysis. We evaluate KBC on five synthetic datasets that represent a diverse range of
cluster geometries and size distributions
(Figure \ref{fig:datasets_and_decomp}, top row; see
Appendix \ref{app:datasets} for the full details).


\begin{figure}[ht]
  \centering
  \includegraphics[width=\textwidth]{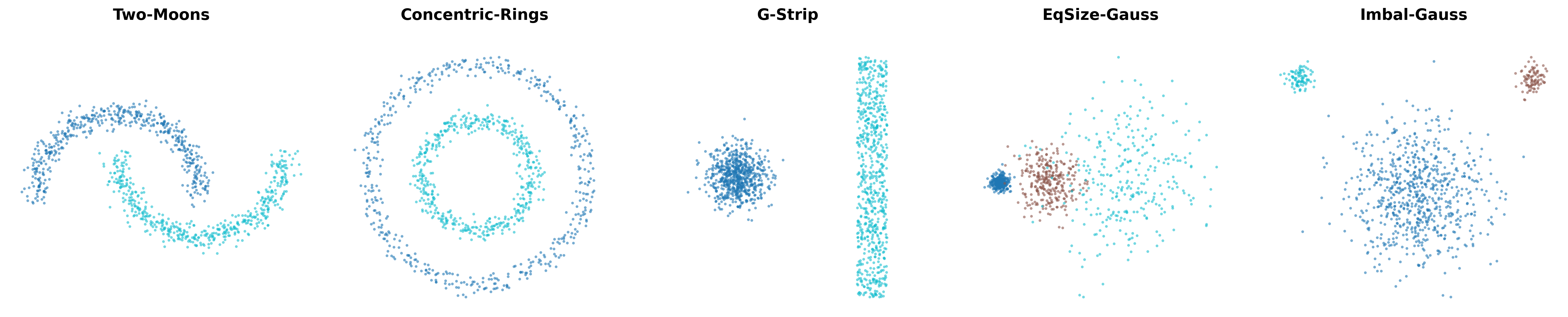}
  \includegraphics[width=\textwidth]{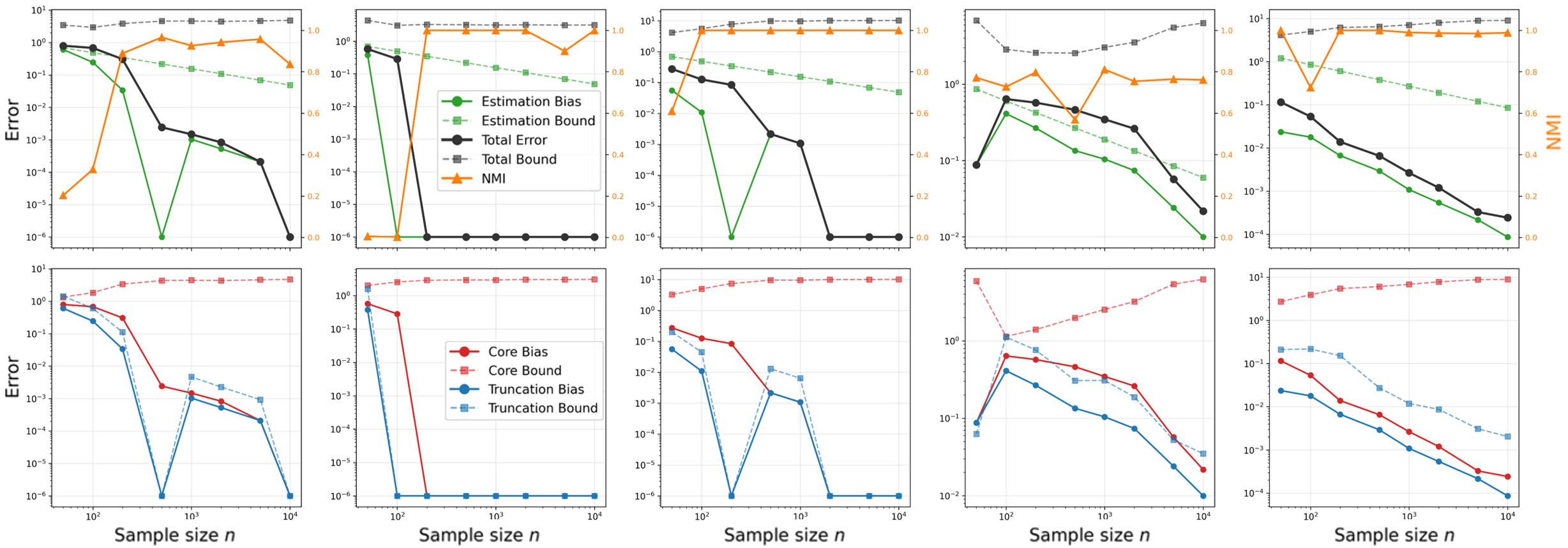}
  \caption{%
    \textit{Top.} Five benchmark datasets.
    \textit{Middle.} Mean (over the number of clusters) estimation bias and total error w.r.t. sample size~$n$; NMI on the right axis.
    \textit{Bottom row.} Core and truncation biases w.r.t. $n$.
    All results are averaged over 10 seeds. The detailed mean and std are reported in Table \ref{tab:std_all}.
  }
  \label{fig:datasets_and_decomp}
\end{figure}

Figure \ref{fig:datasets_and_decomp} (middle and bottom rows) reports all
components on log scale. Middle row shows that both the estimation bias and total error decay steadily
with $n$, and NMI converges toward $1$ across all five datasets, confirming
that KBC is consistent. The bottom row shows the core bias and truncation bias likewise decrease
monotonically. The dashed theoretical bounds remain valid across all tested
sample sizes and correctly reflect the overall decay trend. The core bias
bound is visibly more conservative than the empirical curve. This is expected
because the bound in Lemma~\ref{lemma: cbias} is model-agnostic and therefore
not tailored to the specific core extraction rule used by KBC.

Figure \ref{fig:datasets_and_decomp} serves as a diagnostic validation under an oracle hyperparameter protocol.  Appendix \ref{app_heuristic} provides a heuristic unsupervised method that selects a hyperparameter to form core clusters that yield high
coverage, strong within-core cohesion, and balanced core clusters. The heuristic is shown to provide a  good selection  that produces optimal or close to optimal CaD clustering outcomes.


\subsection{Comparison with Kernel $k$-means}
\label{sec:exp_kkm}

We choose Gaussian Kernel $k$-means (KKM) as the main baseline because it is the most directly comparable method to KBC. Both KBC methods rely on the same Gaussian kernel and perform clustering in RKHS with a similar objective function $\sum_{C\in\mathcal{C}}\sum_{x\in C} f(x,\hat{\mu}_C)$ having a minor difference in $f$: KKM employs squared Euclidean distance and KBC utilizes the dot product. Their two main differences are: (i) Cluster representation: KKM treats $\hat{\mu}_C$ as a mean vector  and it is updated in each iteration of an optimization process. Furthermore, $\hat{\mu}_C$ is a randomly initialized, relying on the optimization to significantly improve $\hat{\mu}_C$ as the iteration progresses. Most importantly, \emph{none of the mean vectors $\hat{\mu}_C$---the initialization and its updates---are treated as a distribution} in the whole clustering process. In contrast, KBC treats $\hat{\mu}_C$ as a distributional proxy of cluster $C$---in the initialization of cluster cores as well as those in the point assignment---in the entire clustering process. (ii) Optimization required: KKM must employ the Expectation-Maximization (EM) optimization, but KBC uses a greedy search in the one-pass assignment---equivalent to the only one iteration in the EM optimization! Therefore, comparing against KKM focuses specifically on the benefits of the distributional representation of core clusters and the point assignment can be achieved in a greedy search in one iteration.

We compare KBC against Kernel $k$-means (KKM) on the same five benchmark
datasets to assess (1) the values of the objective functions KBC and KKM, (2) the embedding approximation
error (w.r.t. the ground truth embedding), and (3) empirical clustering performance in terms of NMI, as the number of iterations increases. Two variants of KKM are also used: (I) \emph{KKM+KBC assign}: KKM completes the ordinary run in each iteration followed by the point assignment step of KBC: each point is reassigned to the cluster $j^*$ to maximize the inner
    product $\langle \phi(x), \hat{\mu}_j \rangle_{\mathcal{H}}$. 
    (II) \emph{KBC-guided KKM}: the KKM updates $\hat{\mu}_j$ in each iteration by employing the labels assigned by the distributional kernel, as used in the second step of KBC. All methods use Gaussian kernel with bandwidth $\sigma$ tuned jointly by a grid
search.
Each of KKM and its two variants runs with $20$ random initializations per dataset, and the best result is
reported for comparisons. Unlike KKM, KBC does not rely on random initialization; its procedure is deterministic and free of stochastic components.

\begin{figure}[ht]
  \centering
\includegraphics[width=\textwidth]{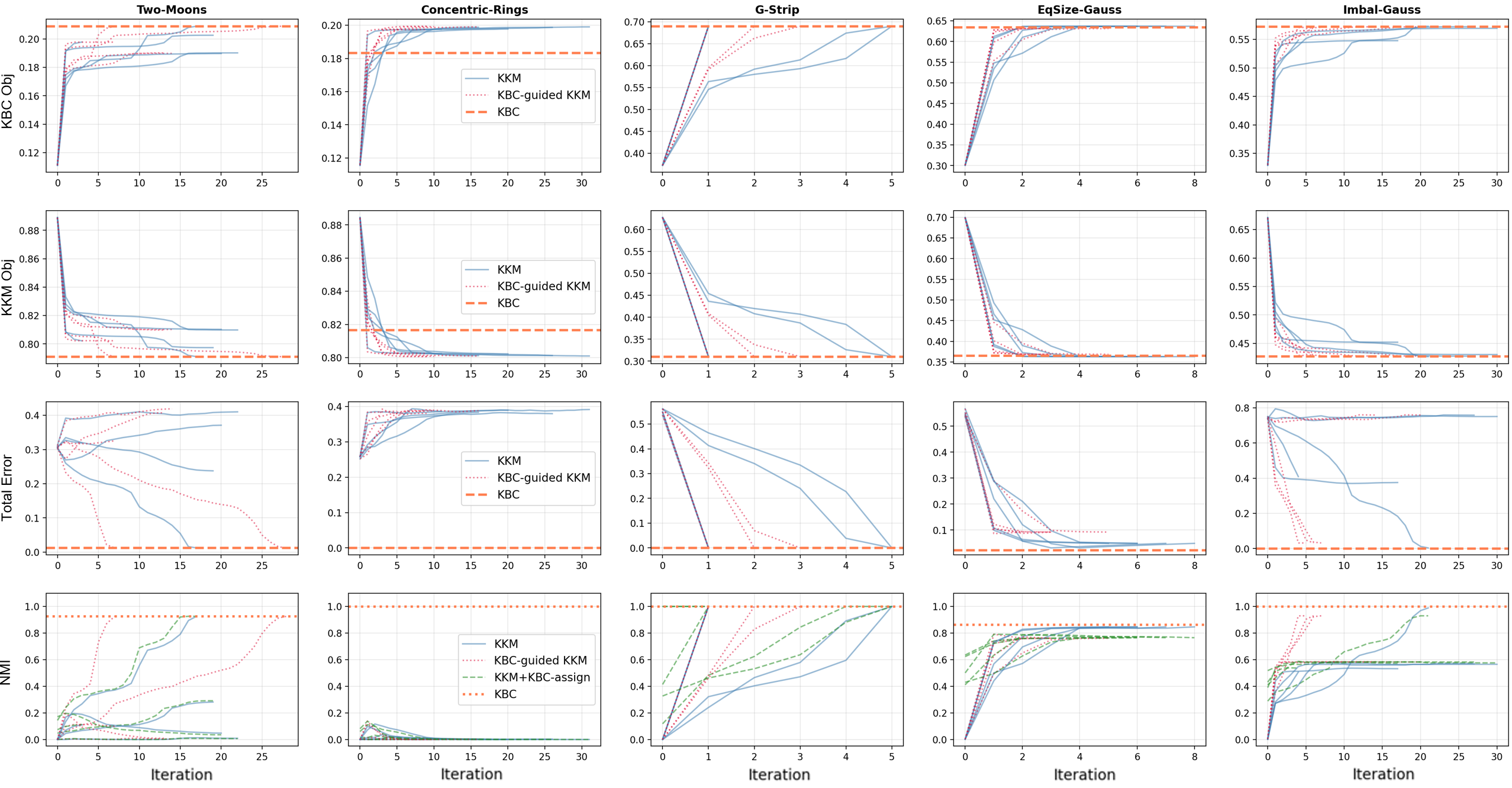}
  \caption{%
    \textbf{Per-iteration dynamics.}
    \textit{Row 1.} KBC objective function value. \textit{Row 2.} KKM objective function value. \textit{Row 3.} Total Error.
    \textit{Row 4.} NMI at each iteration for the same runs.    
  }
  \label{fig:kkm_iter}
\end{figure}

\textbf{The observations of the per-iteration dynamics shown in Figure \ref{fig:kkm_iter} are given as follows:} 
\begin{myitems}
    \item As shown in the first two rows, KBC-guided KKM consistently achieves higher distributional objective (Row 1) and lower within-cluster RKHS difference (Row 2) at each iteration than the vanilla KKM across all five datasets. This demonstrates that the `warm' start of KBC places each iteration in a more favorable region of the objective landscape, enabling faster convergence to better solutions and mitigating the sensitivity to random initialization that afflicts the vanilla KKM. Note that KBC has the best objective values, Total Error and NMI in all datasets, except Concentric-Ring (see the discussion in the next paragraph). 

    \item The result in terms of the total error of the embedding w.r.t. the ground-truth cluster embedding (Row 3) reveals the sharpest distinction: KBC-guided KKM rapidly reduces to near zero across almost all datasets, whereas the vanilla KKM plateaus at substantially higher values on non-convex datasets. This shows that
   optimizing the KKM objective alone does not necessarily recover the ground-truth clusters. Concentric-Rings (Column 2) is an example: although KKM's objective value (Row 2) decreases as the number of iterations increases, the ring-shaped clusters cannot be represented well by the mean vectors, and all versions of KKM converge to large total errors (Row~3) and have very poor NMIs (Row 4). Interestingly, KBC has NMI=1 and Total Error = 0 (perfect clustering outcome), even though it achieves its objective value which is lower than that achieved by all versions of KKM (shown in Row 1). This example demonstrates the power of the `cluster-as-distribution' assumption used in CaD Clustering over the `cluster as a set of similar points' assumption employed in KKM. The visualization of clustering results can be found in Appendix \ref{appendix:cvis}.
\end{myitems}

\section{Discussion}



\textbf{Comparison with more baselines.} Representative CaD clustering methods on real datasets have already been extensively evaluated in prior works. KBC has been compared with spectral clustering and other baselines (KKM, U-SPEC \citep{8661522}, Density Peak Clustering \citep{rodriguez2014clustering}) and their variants (Table A6 in \citep{zhang2025kernel2}), while IDKC/GDKC has been validated on 22 datasets (Table 3 in \citep{zhu2023kernel}). KBC and GDKC both outperform KKM in clustering result and runtime. The purpose of our experiments is 
to provide diagnostics for the proposed error decomposition and to explain why the distributional proxies of core clusters, even with some error, can lead to optimal one-pass assignment via greedy search.


\section{Conclusions}

This is the first analysis, from the statistical and combinatorial viewpoints, that explains why CaD clustering via greedy search can discover clusters of arbitrary shapes, densities and sizes (where all set-oriented clustering methods have failed to discover) when the estimated cluster embeddings faithfully approximate the underlying distributions of the clusters. We decompose the distribution-to-core embedding error into three components: truncation, estimation, and core selection terms, and show that the one-pass assignment via greedy search is optimal for the fixed-core assignment subproblem through a partition matroid formulation. This leads to a near-optimality guarantee for the CaD clustering objective and label recovery, with the distribution regret controlled by the
approximation error. The cluster-as-distribution (CAD) assumption has rendered the greedy-optimality proof via matroid straightforward, while the commonly-held set-oriented view of `cluster is a set of similar points' has made the correct cluster identification impossible for existing clustering algorithms \citep{kleinberg2002impossibility}. This is because it is difficult to ascertain the membership of a cluster based on a point-to-point similarity function, ignoring the distributional information in the dataset; whereas it is straightforward when a cluster is treated as a distribution via a distributional kernel, as shown in the CaD Clustering. Note that the impossibility theorem of clustering \citep{kleinberg2002impossibility} is derived based on the above set-oriented view, and thus it is not applicable to the CaD Clustering.
A further discussion of this issue,
the limitations of this investigation and future work are provided in Appendix \ref{appendix: future-work-and-limitation}.

{
\small
\bibliography{reference}

@article{edmonds1971matroids,
  title={Matroids and the greedy algorithm},
  author={Edmonds, Jack},
  journal={Mathematical Programming},
  volume={1},
  number={1},
  pages={127--136},
  year={1971},
  publisher={Springer}
}

@article{rado1957note,
  title={Note on independence functions},
  author={Rado, Richard},
  journal={Proceedings of the London Mathematical Society},
  volume={3},
  number={1},
  pages={300--320},
  year={1957},
  publisher={Oxford University Press}
}

@article{ting2022point,
  title={Point-set kernel clustering},
  author={Ting, Kai Ming and Wells, Jonathan R and Zhu, Ye},
  journal={IEEE Transactions on Knowledge and Data Engineering},
  volume={35},
  number={5},
  pages={5147--5158},
  year={2023},
  publisher={IEEE}
}

@article{zhu2023kernel,
  title={Kernel-based clustering via isolation distributional kernel},
  author={Zhu, Ye and Ting, Kai Ming},
  journal={Information Systems},
  volume={117},
  pages={102212},
  year={2023},
  publisher={Elsevier}
}

@article{zhang2025kernel,
title = {Kernel-bounded clustering: Achieving the objective of spectral clustering without eigendecomposition},
journal = {Artificial Intelligence},
volume = {350},
pages = {104440},
year = {2026},
issn = {0004-3702},
doi = {https://doi.org/10.1016/j.artint.2025.104440},
url = {https://www.sciencedirect.com/science/article/pii/S0004370225001596},
author = {Hang Zhang and Kai Ming Ting and Ye Zhu}
}

@article{zhang2025kernel2,
  title={Kernel-bounded clustering for spatial transcriptomics enables scalable discovery of complex spatial domains},
  author={Zhang, Hang and Zhang, Yi and Ting, Kai Ming and Zhang, Jie and Zhao, Qiuran},
  journal={Genome Research},
  volume={35},
  number={2},
  pages={355--367},
  year={2025},
  publisher={Cold Spring Harbor Lab}
}

@article{ting2026achieve,
  title={How to Achieve the Intended Aim of Deep Clustering Now, without Deep Learning},
  author={Ting, Kai Ming and Xu, Wei-Jie and Zhang, Hang},
  journal={arXiv preprint arXiv:2602.05749},
  year={2026}
}

@article{macnaughton1964dissimilarity,
  title={Dissimilarity analysis: a new technique of hierarchical sub-division},
  author={Macnaughton-Smith, P and Williams, WT and Dale, MB and Mockett, LG},
  journal={Nature},
  volume={202},
  number={4936},
  pages={1034--1035},
  year={1964},
  publisher={Nature Publishing Group UK London}
}

@book{aggarwal2015data,
  title={Data mining: the textbook},
  author={Aggarwal, Charu C},
  volume={1},
  number={3},
  year={2015},
  publisher={Springer}
}

@inproceedings{ting2018isolation,
  title={Isolation kernel and its effect on {SVM}},
  author={Ting, Kai Ming and Zhu, Yue and Zhou, Zhi-Hua},
  booktitle={Proceedings of the 24th ACM SIGKDD International Conference on Knowledge Discovery \& Data Mining},
  pages={2329--2337},
  year={2018}
}

@article{sriperumbudur2010hilbert,
  title={Hilbert space embeddings and metrics on probability measures},
  author={Sriperumbudur, Bharath K and Gretton, Arthur and Fukumizu, Kenji and Sch{\"o}lkopf, Bernhard and Lanckriet, Gert RG},
  journal={Journal of Machine Learning Research},
  volume={11},
  pages={1517--1561},
  year={2010},
  publisher={JMLR. org}
}

@article{muandet2017kernel,
author = {Muandet, Krikamol and Fukumizu, Kenji and Sriperumbudur, Bharath and Sch\"{o}lkopf, Bernhard},
title = {Kernel Mean Embedding of Distributions: A Review and Beyond},
year = {2017},
issue_date = {Jun 2017},
publisher = {Now Publishers Inc.},
address = {Hanover, MA, USA},
volume = {10},
number = {1–2},
issn = {1935-8237},
url = {https://doi.org/10.1561/2200000060},
doi = {10.1561/2200000060},
journal = {Found. Trends Mach. Learn.},
month = jun,
pages = {1–141},
numpages = {144}
}

@InProceedings{smola2007,
author="Smola, Alex
and Gretton, Arthur
and Song, Le
and Sch{\"o}lkopf, Bernhard",
title="A {H}ilbert Space Embedding for Distributions",
booktitle="Algorithmic Learning Theory",
year="2007",
publisher="Springer",
address="Berlin, Heidelberg",
pages="13--31",
isbn="978-3-540-75225-7"
}

@InProceedings{NN-meaningful-1999,
author="Beyer, Kevin
and Goldstein, Jonathan
and Ramakrishnan, Raghu
and Shaft, Uri",
editor="Beeri, Catriel
and Buneman, Peter",
title="When Is ``Nearest Neighbor'' Meaningful?",
booktitle="Database Theory --- ICDT'99",
year="1999",
publisher="Springer Berlin Heidelberg",
address="Berlin, Heidelberg",
pages="217--235"
}

@article{gretton2006kernel,
  title={A kernel method for the two-sample-problem},
  author={Gretton, Arthur and Borgwardt, Karsten and Rasch, Malte and Sch{\"o}lkopf, Bernhard and Smola, Alex},
  journal={Advances in Neural Information Processing Systems},
  volume={19},
  year={2006}
}

@ARTICLE{8661522,
  author={Huang, Dong and Wang, Chang-Dong and Wu, Jian-Sheng and Lai, Jian-Huang and Kwoh, Chee-Keong},
  journal={IEEE Transactions on Knowledge and Data Engineering}, 
  title={Ultra-Scalable Spectral Clustering and Ensemble Clustering}, 
  year={2020},
  volume={32},
  number={6},
  pages={1212-1226},
  doi={10.1109/TKDE.2019.2903410}}

@article{rodriguez2014clustering,
  title={Clustering by fast search and find of density peaks},
  author={Rodriguez, Alex and Laio, Alessandro},
  journal={Science},
  volume={344},
  number={6191},
  pages={1492--1496},
  year={2014},
  publisher={American Association for the Advancement of Science}
}

@article{zhu2016density,
  title={Density-ratio based clustering for discovering clusters with varying densities},
  author={Zhu, Ye and Ting, Kai Ming and Carman, Mark J},
  journal={Pattern Recognition},
  volume={60},
  pages={983--997},
  year={2016},
  publisher={Elsevier}
}

@article{jain2010data,
  title={Data clustering: 50 years beyond K-means},
  author={Jain, Anil K},
  journal={Pattern Recognition Letters},
  volume={31},
  number={8},
  pages={651--666},
  year={2010},
  publisher={Elsevier}
}

@article{bhattacharjee2021survey,
  title={A survey of density based clustering algorithms},
  author={Bhattacharjee, Panthadeep and Mitra, Pinaki},
  journal={Frontiers of Computer Science},
  volume={15},
  number={1},
  pages={151308},
  year={2021},
  publisher={Springer}
}

@article{nadler2006fundamental,
  title={Fundamental limitations of spectral clustering},
  author={Nadler, Boaz and Galun, Meirav},
  journal={Advances in Neural Information Processing Systems},
  volume={19},
  year={2006}
}

@article{zhou2024comprehensive,
  title={A comprehensive survey on deep clustering: Taxonomy, challenges, and future directions},
  author={Zhou, Sheng and Xu, Hongjia and Zheng, Zhuonan and Chen, Jiawei and Li, Zhao and Bu, Jiajun and Wu, Jia and Wang, Xin and Zhu, Wenwu and Ester, Martin},
  journal={ACM Computing Surveys},
  volume={57},
  number={3},
  pages={1--38},
  year={2024},
  publisher={ACM New York, NY}
}

@article{ros2024deep,
  title={Deep clustering framework review using multicriteria evaluation},
  author={Ros, Frederic and Riad, Rabia and Guillaume, Serge},
  journal={Knowledge-Based Systems},
  volume={285},
  pages={111315},
  year={2024},
  publisher={Elsevier}
}

@article{lu2024survey,
  title={A Survey on Deep Clustering: From the Prior Perspective},
  author={Lu, Yiding and Li, Haobin and Li, Yunfan and Lin, Yijie and Peng, Xi},
  journal={arXiv preprint arXiv:2406.19602},
  year={2024}
}

@article{reynolds2009gaussian,
  title={Gaussian mixture models.},
  author={Reynolds, Douglas A and others},
  journal={Encyclopedia of biometrics},
  volume={741},
  number={659-663},
  pages={3},
  year={2009},
  publisher={Springer City}
}

@inproceedings{ester1996density,
  title={A density-based algorithm for discovering clusters in large spatial databases with noise},
  author={Ester, Martin and Kriegel, Hans-Peter and Sander, J{\"o}rg and Xu, Xiaowei},
  booktitle={Proceedings of the Second International Conference on Knowledge Discovery and Data Mining},
  pages={226--231},
  year={1996}
}

@article{faigle2009general,
  title={A general model for matroids and the greedy algorithm},
  author={Faigle, Ulrich and Fujishige, Satoru},
  journal={Mathematical Programming},
  volume={119},
  number={2},
  pages={353--369},
  year={2009},
  publisher={Springer}
}

@article{kleinberg2002impossibility,
  title={An impossibility theorem for clustering},
  author={Kleinberg, Jon},
  journal={Advances in Neural Information Processing Systems},
  volume={15},
  year={2002}
}

@article{ting2024possible,
  title={Is it possible to find the single nearest neighbor of a query in high dimensions?},
  author={Ting, Kai Ming and Washio, Takashi and Zhu, Ye and Xu, Yang and Zhang, Kaifeng},
  journal={Artificial Intelligence},
  volume={336},
  pages={104206},
  year={2024},
  publisher={Elsevier}
}

@article{gormley2023model,
  title={Model-based clustering},
  author={Gormley, Isobel Claire and Murphy, Thomas Brendan and Raftery, Adrian E},
  journal={Annual Review of Statistics and Its Application},
  volume={10},
  number={1},
  pages={573--595},
  year={2023},
  publisher={Annual Reviews}
}

@inproceedings{dhillon2004kernel,
  title={Kernel k-means: spectral clustering and normalized cuts},
  author={Dhillon, Inderjit S and Guan, Yuqiang and Kulis, Brian},
  booktitle={Proceedings of the tenth ACM SIGKDD International Conference on Knowledge Discovery and Data Mining},
  pages={551--556},
  year={2004}
}

@article{von2007tutorial,
  title={A tutorial on spectral clustering},
  author={Von Luxburg, Ulrike},
  journal={Statistics and Computing},
  volume={17},
  number={4},
  pages={395--416},
  year={2007},
  publisher={Springer}
}

@article{wu2022deep,
  title={Deep clustering and visualization for end-to-end high-dimensional data analysis},
  author={Wu, Lirong and Yuan, Lifan and Zhao, Guojiang and Lin, Haitao and Li, Stan Z},
  journal={IEEE Transactions on Neural Networks and Learning Systems},
  volume={34},
  number={11},
  pages={8543--8554},
  year={2022},
  publisher={IEEE}
}

@article{rinaldo2012stability,
  title={Stability of Density-Based Clustering.},
  author={Rinaldo, Alessandro and Singh, Aarti and Nugent, Rebecca and Wasserman, Larry},
  journal={Journal of Machine Learning Research},
  volume={13},
  number={4},
  year={2012}
}
\bibliographystyle{plain}
}


\appendix

\section{Notations}
\label{appendix: notation}
We show the main notations used in the paper in Table \ref{tab:notation}.
\begin{table}[ht]
\centering
\caption{Notations used for basic symbols, clustering, and matroid analysis.}
\label{tab:notation}
\small
\begin{tabular}{cc}
\toprule
\textbf{Notation} & \textbf{Description} \\
\midrule
$D$ & Finite dataset $\{x_1,\dots,x_n\}$. \\
$\mathcal{H}$ & Reproducing Kernel Hilbert Space. \\
$\phi(x)$ & Feature map of point $x$ in $\mathcal{H}$. \\
$\mathcal{P}$ & Distribution of a cluster. \\
$\mathcal{Q}$ & Truncated distribution of $\mathcal{P}$, supported on $S_\lambda=\{x\in \text{supp}(\mathcal{P})|f_\mathcal{P} (x)\geq\lambda\}$ \\
$G$ & Core cluster, a representative subset of points used to characterize a cluster. \\
$\mu_\mathcal{P}$ & Kernel Mean Embedding of distribution $\mathcal{P}$. \\
$\hat{\mu}_G$ & Estimated Kernel Mean Embedding of a dataset $G$ with finite points. \\
\midrule
$\mathcal{I}$ & Family of independent sets in the partition matroid. \\
$w(x_i,j)$ & Weight of assigning point $x_i$ to cluster $j$. \\
\bottomrule
\end{tabular}
\end{table}

\section{Proofs}
\label{appendix: proofs}

\subsection{Proof of Lemma \ref{lemma: tbias}}

\textbf{Lemma \ref{lemma: tbias}:}

    When using Gaussian Kernel $k_\sigma(x, y)$ and $\sigma\rightarrow\infty$, for the truncation bias, it holds that 

\begin{equation}
\|\mu_\mathcal{P} - \mu_\mathcal{Q}\|_{\mathcal{H}}
= \|\mathbb{E}[\mathcal{Q}]-\mathbb{E}[\mathcal{T}]\|\cdot \frac{\eta}{\sigma} + O\!\left(\frac{1}{\sigma^2}\right).
\end{equation}

\begin{proof}
\noindent Expanding the inner product $\langle \cdot, \cdot \rangle_{\mathcal{H}}$, we have 
\begin{equation}
    \|\mathcal{E}\|_{\mathcal{H}}^2 = \eta^2 \left( \langle \mu_\mathcal{T}, \mu_\mathcal{T} \rangle_{\mathcal{H}} + \langle \mu_\mathcal{Q}, \mu_\mathcal{Q} \rangle_{\mathcal{H}} - 2\langle \mu_\mathcal{T}, \mu_\mathcal{Q} \rangle_{\mathcal{H}} \right).
\end{equation}
Using the reproducing property $\langle \mu_\mathcal{P}, \mu_\mathcal{Q} \rangle_{\mathcal{H}} = \mathbb{E}_{x \sim \mathcal{P}, y \sim \mathcal{Q}}[k(x, y)]$, we express the bound in terms of Expected Kernels:
\begin{equation}
    \|\mathcal{E}\|_{\mathcal{H}}^2 = \eta^2 \left( \underbrace{\mathbb{E}_{x,x' \sim \mathcal{T}}[k(x,x')]}_{\gamma_\mathcal{T}} + \underbrace{\mathbb{E}_{y,y' \sim \mathcal{Q}}[k(y,y')]}_{\gamma_\mathcal{Q}} - 2\underbrace{\mathbb{E}_{x \sim \mathcal{T}, y \sim \mathcal{Q}}[k(x,y)]}_{\delta_{\mathcal{T}\mathcal{Q}}} \right),
\end{equation}
where $\gamma_\mathcal{T}, \gamma_\mathcal{Q}$ represent the self-similarity (intra-cluster coherence) and $\delta_{\mathcal{T}\mathcal{Q}}$ represents the cross-similarity between $\mathcal{Q}$ and $\mathcal{T}$.

For a Gaussian kernel, $k_\sigma(x, y) \in (0, 1]$. As $\sigma \to \infty$, the second-order Taylor expansion with remainder gives
\begin{equation}
    k_\sigma(x,y) = 1 - \frac{\|x-y\|^2}{2\sigma^2} + O\!\left(\frac{\|x-y\|^4}{\sigma^4}\right).
\end{equation}
Substituting into the RKHS norm expression yields
\begin{equation}
    \|\mathcal{E}\|_{\mathcal{H}}^2 = \eta^2 \left [ \left(1 - \frac{\mathbb{E}\|x-x'\|^2}{2\sigma^2}\right) + \left(1 - \frac{\mathbb{E}\|y-y'\|^2}{2\sigma^2}\right) - 2\left(1 - \frac{\mathbb{E}\|x-y\|^2}{2\sigma^2}\right) \right] + O\!\left(\frac{1}{\sigma^4}\right).
\end{equation}
Simplifying the terms:
\begin{equation}
\begin{aligned}
    \|\mathcal{E}\|_{\mathcal{H}}^2
    &= \eta^2 \left[ \frac{2\mathbb{E}_{x \in \mathcal{T}, y \in \mathcal{Q}}\|x-y\|^2 - \mathbb{E}_{x,x' \in \mathcal{T}}\|x-x'\|^2 - \mathbb{E}_{y,y' \in \mathcal{Q}}\|y-y'\|^2}{2\sigma^2} \right] + O\!\left(\frac{1}{\sigma^4}\right) \\
    &= \eta^2\frac{\|\mathbb{E}[\mathcal{Q}]-\mathbb{E}[\mathcal{T}]\|^2}{\sigma^2} + O\!\left(\frac{1}{\sigma^4}\right).
\end{aligned}
\end{equation}

Now let
\begin{equation}
    A=\eta^2\|\mathbb{E}[\mathcal{Q}]-\mathbb{E}[\mathcal{T}]\|^2,
    \qquad
    f(\sigma)=\|\mu_\mathcal{P} - \mu_\mathcal{Q}\|_{\mathcal{H}}\ge 0.
\end{equation}
From the previous step,
\begin{equation}
    f(\sigma)^2 = \frac{A}{\sigma^2} + O\!\left(\frac{1}{\sigma^4}\right),
\end{equation}
so there exists $C>0$ such that, for sufficiently large $\sigma$,
\begin{equation}
    \left|f(\sigma)^2-\frac{A}{\sigma^2}\right|\le \frac{C}{\sigma^4}.
\end{equation}
Define $g(\sigma)=\sqrt{A}/\sigma\ge 0$. Then
\begin{equation}
    |f(\sigma)-g(\sigma)|
    = \frac{|f(\sigma)^2-g(\sigma)^2|}{f(\sigma)+g(\sigma)}.
\end{equation}
Moreover, since $f(\sigma)^2=A/\sigma^2+O(1/\sigma^4)$, we have $f(\sigma)+g(\sigma)=\Theta(1/\sigma)$ in the non-degenerate case $A>0$. Therefore
\begin{equation}
    |f(\sigma)-g(\sigma)|
    = O\!\left(\frac{1/\sigma^4}{1/\sigma}\right)
    = O\!\left(\frac{1}{\sigma^3}\right)
    \subseteq O\!\left(\frac{1}{\sigma^2}\right).
\end{equation}
Hence the conservative asymptotic form is
\begin{equation}
    \|\mu_\mathcal{P} - \mu_\mathcal{Q}\|_{\mathcal{H}} = \frac{\sqrt{A}}{\sigma} + O\!\left(\frac{1}{\sigma^2}\right).
\end{equation}

Thus, we have
$$
\|\mu_\mathcal{P} - \mu_\mathcal{Q}\|_{\mathcal{H}}
= \|\mathbb{E}[\mathcal{Q}]-\mathbb{E}[\mathcal{T}]\|\cdot \frac{\eta}{\sigma} + O\!\left(\frac{1}{\sigma^2}\right).
$$
\end{proof}

\subsection{Proof of Lemma \ref{lemma ebias}}

\textbf{Lemma \ref{lemma ebias} (Convergence of the Embedding of Dense Region ):}

For a sample of $s$ i.i.d. points $U$ from distribution $\mathcal{Q}$ and a Gaussian kernel with $k(x,x) \le 1$, it holds that with probability at least $1-\delta$
\begin{equation}
    \|\KME_\mathcal{Q} - \empKME_U\|_{\Hspace} \le \frac{1 + \sqrt{2 \ln(1/\delta)}}{\sqrt{s}}.
\end{equation}

\begin{proof}
Consider the function $f(x_1, \dots, x_s) = \|\KME_\mathcal{Q} - \frac{1}{s} \sum_{i=1}^s \phiMap(x_i)\|_{\Hspace}$. 
Changing one coordinate $x_i$ to $x_i'$ results in a maximum change:
\begin{equation}
    |f(\dots, x_i, \dots) - f(\dots, x_i', \dots)| \le \frac{1}{s} \|\phiMap(x_i) - \phiMap(x_i')\|_{\Hspace} \le \frac{2}{s}
\end{equation}
Applying McDiarmid's inequality:
\begin{equation}
    P(f - \E[f] \ge \epsilon) \le \exp\left( -\frac{2\epsilon^2}{\sum (2/s)^2} \right) = \exp\left( -\frac{s\epsilon^2}{2} \right)
\end{equation}
Setting $\delta = \exp(-s\epsilon^2/2)$ gives the deviation term $\epsilon=\sqrt{\frac{2\ln(1/\delta)}{s}}$. Then with probability at least $1-\delta$, it holds 
\begin{equation}
f-\mathbb{E}[f]=
    \|\mu_\mathcal{Q} - \hat{\mu}_U\|_{\mathcal{H}} - \mathbb{E}[\|\mu_\mathcal{Q} - \hat{\mu}_U\|_{\mathcal{H}}]\le \sqrt{\frac{2 \ln(1/\delta)}{s}}.
\end{equation}

By Jensen's inequality, the expected value is bounded by:
\begin{equation}
    \E[f] \le \sqrt{\E[f^2]} = \sqrt{\frac{1}{s^2} \sum _{i=1}^{s}\text{Var}(\phiMap(x_i))} \le \frac{1}{\sqrt{s}}.
\end{equation}
This holds because given $\|\phi(x)\| \le 1$, we have $\text{Var}_\mathcal{Q}(\phi(X)) \le \mathbb{E}\|\phi(x)\|^2 \le 1$.

Combining the two inequalities above, we finish the proof.
\end{proof}

\subsection{Proof of Lemma \ref{lemma: cbias}}

\textbf{Lemma \ref{lemma: cbias} (A generic upper bound for core bias):}

Let $U=\{q_i\}_{i=1}^s$ be sampled points from $\mathcal{Q}$ and let $G=\{g_j\}_{j=1}^m$ be core points. Assume there exists a mapping $\pi:\{1,\dots,s\}\to\{1,\dots,m\}$ such that
\begin{equation}
    \|q_i-g_{\pi(i)}\|\le r, \quad \forall i\in\{1,\dots,s\}.
\end{equation}
Define occupancy proportions
\begin{equation}
    p_j=\frac{1}{s}\left|\{i:\pi(i)=j\}\right|,\quad j=1,\dots,m,
\end{equation}
and imbalance term
\begin{equation}
    B=\sum_{j=1}^m\left|p_j-\frac{1}{m}\right|.
\end{equation}
Then for Gaussian kernel $k_\sigma$,
\begin{equation}
    \|\empKME_U-\empKME_G\|_{\mathcal{H}}
    \le \frac{r}{\sigma} + B.
\end{equation}

\begin{proof}
Let
\begin{equation}
    \widetilde{\mu}_G = \frac{1}{s}\sum_{i=1}^s\phiMap(g_{\pi(i)}) = \sum_{j=1}^m p_j\phiMap(g_j).
\end{equation}
By triangle inequality,
\begin{equation}
    \|\empKME_U-\empKME_G\|_{\mathcal{H}}
    \le \|\empKME_U-\widetilde{\mu}_G\|_{\mathcal{H}} + \|\widetilde{\mu}_G-\empKME_G\|_{\mathcal{H}}.
\end{equation}
For the first term,
\begin{equation}
\|\empKME_U-\widetilde{\mu}_G\|_{\mathcal{H}}
=\left\|\frac{1}{s}\sum_{i=1}^s\left(\phiMap(q_i)-\phiMap(g_{\pi(i)})\right)\right\|_{\mathcal{H}}
\le \frac{1}{s}\sum_{i=1}^s\|\phiMap(q_i)-\phiMap(g_{\pi(i)})\|_{\mathcal{H}}.
\end{equation}
For Gaussian kernel, using $k(x,x)=1$,
\begin{equation}
\|\phiMap(x)-\phiMap(y)\|_{\mathcal{H}}^2 = 2\left(1-e^{-\|x-y\|^2/(2\sigma^2)}\right) \le \frac{\|x-y\|^2}{\sigma^2},
\end{equation}
thus $\|\phiMap(x)-\phiMap(y)\|_{\mathcal{H}}\le \|x-y\|/\sigma$. Substituting $\|q_i-g_{\pi(i)}\|\le r$ for all $i$ gives
\begin{equation}
\|\empKME_U-\widetilde{\mu}_G\|_{\mathcal{H}}\le \frac{r}{\sigma}.
\end{equation}
For the second term,
\begin{equation}
\|\widetilde{\mu}_G-\empKME_G\|_{\mathcal{H}}
=\left\|\sum_{j=1}^m\left(p_j-\frac{1}{m}\right)\phiMap(g_j)\right\|_{\mathcal{H}}
\le \sum_{j=1}^m\left|p_j-\frac{1}{m}\right|\|\phiMap(g_j)\|_{\mathcal{H}}
\le B,
\end{equation}
where $\|\phiMap(g_j)\|_{\mathcal{H}}=\sqrt{k(g_j,g_j)}=1$. Combining the two bounds yields
\begin{equation}
\|\empKME_U-\empKME_G\|_{\mathcal{H}}\le \frac{r}{\sigma}+B.
\end{equation}
\end{proof}

\subsection{Proof of Theorem \ref{theorem: total error}}

\begin{proof}
Starting from the decomposition
\[
\|\KME_\mathcal{P} - \empKME_G\|_{\Hspace}
\le \|\KME_\mathcal{P} - \KME_\mathcal{Q}\|_{\Hspace}
+ \|\KME_\mathcal{Q} - \empKME_U\|_{\Hspace}
+ \|\empKME_U - \empKME_G\|_{\Hspace},
\]
apply Lemma \ref{lemma: tbias}, Lemma \ref{lemma ebias}, and Lemma \ref{lemma: cbias} to the three terms, respectively:
\begin{align}
\|\KME_\mathcal{P} - \KME_\mathcal{Q}\|_{\Hspace}
&\le \|\mathbb{E}[\mathcal{Q}]-\mathbb{E}[\mathcal{T}]\|\cdot \frac{\eta}{\sigma} + O\!\left(\frac{1}{\sigma^2}\right),\\
\|\KME_\mathcal{Q} - \empKME_U\|_{\Hspace}
&\le \frac{1 + \sqrt{2 \ln(1/\delta)}}{\sqrt{s}},\\
\|\empKME_U - \empKME_G\|_{\Hspace}
&\le \frac{r}{\sigma}+B.
\end{align}
Summing the three inequalities gives the stated bound.
\end{proof}

\subsection{Proof of Proposition \ref{prop:matroid_structure}}
\label{appendix:proof-matroid-structure}

\begin{proof}
We verify the three axioms:
\begin{enumerate}
    \item \textbf{Non-emptiness:} $\emptyset \cap S_i = \emptyset$, thus $|\emptyset \cap S_i| = 0 \le 1$. $\emptyset \in \mathcal{I}$.
    \item \textbf{Hereditary:} Let $A \in \mathcal{I}$ and $B \subseteq A$. For any $i$, $|B \cap S_i| \le |A \cap S_i| \le 1$. Thus $B \in \mathcal{I}$.
    \item \textbf{Exchange Property:} Let $A, B \in \mathcal{I}$ with $|A| < |B|$. Let $I(A) = \{i \mid A \cap S_i \neq \emptyset\}$ be the indices of points covered by $A$. Since $|A \cap S_i| \le 1$, $|I(A)| = |A|$. Since $|B| > |A|$, there must exist an index $m \in I(B) \setminus I(A)$. This means point $x_m$ is covered by $B$ but not by $A$. Let $e = (x_m, j) \in B \cap S_m$. Since $m \notin I(A)$, $A \cap S_m = \emptyset$, so $|(A \cup \{e\}) \cap S_m| = 1$. For all other $i \neq m$, $|(A \cup \{e\}) \cap S_i| = |A \cap S_i| \le 1$. Thus $A \cup \{e\} \in \mathcal{I}$.
\end{enumerate}
\end{proof}



\subsection{Proof of Theorem \ref{theorem: fixed_core_opt}}
\label{appendix:proof-fixed-core-opt}

\begin{proof}
The one-pass rule is a greedy algorithm on the partition matroid $(\mathcal{S}_{\mathrm{g}},\mathcal I)$ with additive weights $w(x_i,j)$. The feasible assignments with $|A|=n$ are exactly the bases of this partition matroid. Therefore, by the Rado-Edmonds theorem, greedy returns a maximum-weight basis, which is globally optimal for
\begin{equation}
\max_{A\in\mathcal I,\,|A|=n}\sum_{e\in A}w(e).
\end{equation}
\end{proof}

\subsection{Proof of Proposition \ref{prop:near_opt_population_obj}}
\label{appendix:proof-near-opt-population}

\begin{proof}
For any pair $(x_i,j)$,
\begin{equation}
    |w^\star(x_i,j)-\hat w(x_i,j)|
    =|\langle \phiMap(x_i),\mu_j-\hat\mu_j\rangle_{\Hspace}|
    \le \|\phiMap(x_i)\|_{\Hspace}\,\|\mu_j-\hat\mu_j\|_{\Hspace}
    \le \varepsilon.
\end{equation}
Hence for any feasible $A$ with $|A|=n$,
\begin{equation}
    |W^\star(A)-\hat W(A)|
    \le \sum_{(x_i,j)\in A}|w^\star(x_i,j)-\hat w(x_i,j)|
    \le n\varepsilon.
\end{equation}
Now decompose
\begin{align}
    W^\star(A^\star)-W^\star(\hat A)
    &= \bigl(W^\star(A^\star)-\hat W(A^\star)\bigr)
    +\bigl(\hat W(A^\star)-\hat W(\hat A)\bigr)
    +\bigl(\hat W(\hat A)-W^\star(\hat A)\bigr) \\
    &\le n\varepsilon + 0 + n\varepsilon
    = 2n\varepsilon,
\end{align}
where the middle term is non-positive because $\hat A$ maximizes $\hat W$ over $\{A\in\mathcal I:|A|=n\}$. This proves the claim.
\end{proof}

\subsection{Proof of Proposition \ref{cor:regret_decomposed_terms}}
\label{appendix:proof-regret-decomposed}

\begin{proof}
By definition, $\varepsilon=\max_j\|\mu_j-\hat\mu_j\|_{\Hspace}\le \max_j b_j$. Substituting this into Proposition \ref{prop:near_opt_population_obj} gives the result.
\end{proof}

\subsection{Proof of Theorem \ref{thm:exact_label_recovery}}
\label{appendix:proof-exact-label-recovery}

\begin{proof}
Fix any point $x_i$ and any cluster index $j$. By Cauchy-Schwarz and $\|\phiMap(x_i)\|_{\Hspace}\le 1$,
\begin{equation}
|\hat w(x_i,j)-w^\star(x_i,j)|
=|\langle\phiMap(x_i),\hat\mu_j-\mu_j\rangle_{\Hspace}|
\le \|\phiMap(x_i)\|_{\Hspace}\,\|\hat\mu_j-\mu_j\|_{\Hspace}
\le \varepsilon.
\end{equation}
Hence for the best class $y_i^\star$,
\begin{equation}
\hat w(x_i,y_i^\star)\ge w^\star(x_i,y_i^\star)-\varepsilon,
\end{equation}
while for any $j\neq y_i^\star$,
\begin{equation}
\hat w(x_i,j)\le w^\star(x_i,j)+\varepsilon.
\end{equation}
Therefore
\begin{equation}
\hat w(x_i,y_i^\star)-\hat w(x_i,j)
\ge \big(w^\star(x_i,y_i^\star)-w^\star(x_i,j)\big)-2\varepsilon
\ge \gamma_\star-2\varepsilon > 0.
\end{equation}
So $y_i^\star$ remains the unique maximizer of $\hat w(x_i,\cdot)$, i.e., $\hat y_i=y_i^\star$. Since $i$ is arbitrary, this holds for all points, proving exact recovery.
\end{proof}

\subsection{Proof of Proposition \ref{prop:exact_recovery_via_decomposition}}
\label{appendix:proof-exact-recovery-via-decomposition}

\begin{proof}
From Proposition \ref{cor:regret_decomposed_terms}, for every cluster index $j$,
\begin{equation}
\|\mu_j-\hat\mu_j\|_{\Hspace}\le b_j,
\end{equation}
thus
\begin{equation}
\varepsilon:=\max_{j\in\mathcal K}\|\mu_j-\hat\mu_j\|_{\Hspace}
\le \max_{j\in\mathcal K} b_j.
\end{equation}
If $\max_j b_j < \gamma_\star/2$, then $\varepsilon < \gamma_\star/2$. Applying Theorem \ref{thm:exact_label_recovery} yields $\hat y_i=y_i^\star$ for all $i$.
\end{proof}

\section{Algorithmic Details of CaD Clustering Methods}
\label{appendix:algorithm-details}

This section details how Kernel-Bounded Clustering (KBC) \citep{zhang2025kernel,zhang2025kernel2}, Isolation Distributional Kernel Clustering (IDKC) \citep{zhu2023kernel}, and Point-Set Kernel Clustering (psKC) \citep{ting2022point} instantiate the same high-level pipeline used in our analysis: (i) construct cluster-as-distribution proxies in RKHS, and (ii) assign each point by maximizing similarity (distributional kernel) to one of these proxies. Despite their differences in proxy construction, KBC, psKC, and IDKC share a fundamental common structure: each method assigns data points based on distributional kernel similarity. 

\subsection{KBC (Kernel-Bounded Clustering)}
\begin{figure}[h]
    \centering
\includegraphics[width=0.98\linewidth]{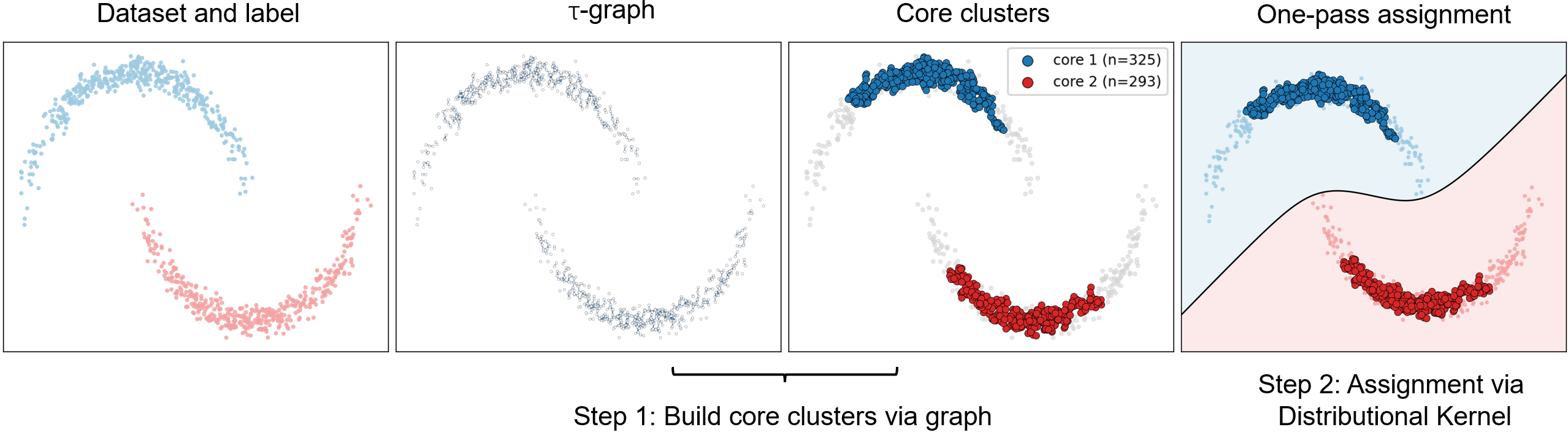}
    \caption{Demonstrate on how KBC works on Two-Moons dataset.}
    \label{fig:kbc demo}
\end{figure}
KBC instantiates the cluster-as-distribution clustering pipeline through
core extraction followed by one-pass assignment, as follows: 

\begin{enumerate}
    \item Build a kernel-threshold graph on a subsample of the data and extract representative high-density connected components as cores
    $G_1,\dots,G_k$.
    \item Compute their empirical embeddings
    \[
        \hat{\mu}_{G_j} = \frac{1}{|G_j|}\sum_{x \in G_j}\phi(x),
    \]
    And assign each point $x$ in the given dataset to
        $\arg\max_j \langle \phi(x), \hat{\mu}_{G_j}\rangle_{\mathcal H}$.
\end{enumerate}

The details are given in Algorithm \ref{alg:kbc} and Figure \ref{fig:kbc demo}.

\begin{algorithm}[ht]
\caption{Kernel Bounded Clustering}
\label{alg:kbc}
\begin{minipage}{0.96\linewidth}
\textbf{Input:} $D$ - dataset, $k$ - number of clusters, $s$ - sample size, $\tau$ - similarity threshold

\textbf{Output:} $\mathcal{C}=\{C_1,\ldots,C_k\}$

\vspace{0.4em}
\noindent\textbf{Step 1: Find the $k$ initial clusters}

\noindent From $D_s \subset D$, find the largest $k$ initial clusters $G_j$ via kernel $\kappa$ as follows: $\forall j\in[k]$,
\[
G_j=\{\mathbf{x},\mathbf{y}\in D \mid \text{there exists a chain } \mathbf{z}_1,\mathbf{z}_2,\cdots,\mathbf{z}_q;
\ \mathbf{z}_1=\mathbf{x},\ \mathbf{z}_q=\mathbf{y},\ \forall i,\ k_{\sigma}(\mathbf{z}_i,\mathbf{z}_{i+1})>\tau\}.
\]

\noindent if the number of clusters in $\{G_1,G_2,\ldots,G_j\}<k$ then

\hspace*{1.5em}Assert `Parameter $\tau$ is set too small!' and Exit

\noindent end if

\vspace{0.4em}
\noindent\textbf{Step 2: Assignment of data points to clusters}
\[
C_j=\{\mathbf{x}\in D \mid \arg\max_{i\in[1,k]} K(\delta(\mathbf{x}),\mathcal{P}_{G_i})=j\},\ \forall j\in[1,k].
\]

\vspace{0.4em}
\noindent{\color{gray}\textbf{Step 3: Refine $\mathcal{C}=\{C_1,\ldots,C_k\}$ to improve the objective:}}
\[
{\color{gray}\max_{\mathcal{C}} \sum_{C\in\mathcal{C}} \sum_{\mathbf{x}\in C} K(\delta(\mathbf{x}),\mathcal{P}_C).}
\]

\vspace{0.4em}
\noindent\textbf{return} $\mathcal{C}$
\end{minipage}
\end{algorithm}

\subsection{psKC and IDKC}
Both psKC and IDKC follow essentially the same cluster-as-distribution clustering process as KBC:
cluster-level proxies are first constructed in RKHS, and each point in the given dataset is then
assigned according to the maximum similarity to one of these proxies.
Their differences lie mainly in how the proxy is initialized and constructed.
psKC identifies the maximum-similarity point with respect to the dataset of all unassigned points as the seed for each proxy, and then grows each proxy via the same distributional kernel in a sequential order. IDKC identifies all $k$ seeds for $k$ proxies and then assign points to form $k$ clusters in parallel via the same distributional kernel.

Hence, from the perspective of our analysis, all psKC, IDKC and KBC \emph{employ the
same point assignment procedure to perform the clustering via a distributional kernel to achieve the same clustering objective}. The main difference is in the
cluster core/proxy initialization.

\subsection{Connection to Our Error Analysis}

Although the three methods differ in how they instantiate kernels and core/cluster representatives, they all use the same assignment based on a distributional kernel, and they all fit the decomposition analyzed in this paper:
\begin{equation}
    \|\mu_\mathcal{P}-\hat\mu_G\|_{\mathcal H}
    \leq
    \|\mu_\mathcal{P}-\mu_\mathcal{Q}\|_{\mathcal H}
    +
    \|\mu_\mathcal{Q}-\hat\mu_U\|_{\mathcal H}
    +
    \|\hat\mu_U-\hat\mu_G\|_{\mathcal H}.
\end{equation}
Therefore, the truncation, estimation, and core-selection terms provide a common lens for understanding approximation quality across KBC, psKC, and IDKC.
\section{Additional Notes}


\subsection{The rationale behind $|A\cap S_i|\leq 1$}
\label{appendix: mr}

$|A\cap S_i|\leq 1$ indicates that a point can be assigned to one cluster or no cluster. But in practice, each point should be assigned to just one cluster. The reason that $|A\cap S_i|\leq 1$ is valid lies in the interaction between the weight function (in Section \ref{sect: opt assignment}) and the global objective. The weight $w(x_i,j)=\langle \phi(x_i), \hat{\mu}_{G_j}\rangle_{\mathcal H}$ evaluates the compatibility between point $x_i$ and cluster $j$, given fixed-core cluster embeddings. The overall objective $W(A)=\sum_{e\in A} w(e)$ is then purely additive over assignment elements, with no cross-point interaction terms. As a result, once the feasible set is constrained by the partition matroid structure, global optimization reduces to selecting exactly one maximum-weight element from each block $S_i$. So in practice, $\forall i\in [n], |A\cap S_i|=1$, i.e., $|A|=n$. This is precisely why the one-pass assignment is not merely heuristic but coincides with the maximum weight basis for fixed core clusters. 

%

\subsection{Core Bias Analysis for KBC}
\label{app:kbc cbias}

Here we give a KBC interpretation of the core bias term. The
generic Lemma~\ref{lemma: cbias} in the main text bounds the core bias through a
sample-to-core coupling radius. Here we use the specific structure of KBC:
for a Gaussian kernel, thresholding kernel values is equivalent to building a
neighborhood graph, and KBC selects large connected components of that graph.
This controls the core bias term by the mass of a boundary
band of a density superlevel set.

\subsubsection{Reminder of our setup}

Let \(\mathcal{P}_j\) be the distribution of cluster \(j\), with density \(f_j\)
on \(\mathbb R^d\). For a density threshold \(\lambda>0\), define the density
superlevel set
\[
S_{j,\lambda}=\{x\in \mathbb R^d: f_j(x)\ge \lambda\}.
\]
Let \(\mathcal{Q}_j=\mathcal{P}_j(\cdot\mid S_{j,\lambda})\), and let
\[
U_j=\{q_1,\ldots,q_{s_j}\}\sim \mathcal{Q}_j^{s_j}
\]
be dense region samples.

For the Gaussian kernel
\[
k_\sigma(x,y)=\exp\!\left(-\frac{\|x-y\|^2}{2\sigma^2}\right),
\]
the KBC threshold condition \(k_\sigma(x,y)>\tau\) is equivalent to
\[
\|x-y\|<\varepsilon_\tau,
\qquad
\varepsilon_\tau=\sigma\sqrt{2\log(1/\tau)}.
\]
Thus KBC thresholding with \(\tau\) is equivalent, for the Gaussian kernel, to
building an \(\varepsilon_\tau\)-neighborhood graph.

Let \(G_j(\tau)\subseteq U_j\) denote the empirical core component selected by
KBC for cluster \(j\), restricted to the dense region samples \(U_j\). Define
the omitted dense region samples
\[
R_j(\tau)=U_j\setminus G_j(\tau),
\]
and the empirical embeddings
\[
\hat\mu_{U_j}=\frac{1}{s_j}\sum_{q\in U_j}\phi(q),
\qquad
\hat\mu_{G_j(\tau)}=\frac{1}{|G_j(\tau)|}\sum_{g\in G_j(\tau)}\phi(g),
\]
where \(\phi(x)=k_\sigma(x,\cdot)\) is the RKHS feature map.

\subsubsection{Boundary Band and KBC Coverage}

\begin{figure}[h]
\centering
\begin{tikzpicture}[scale=0.88, every node/.style={font=\small}]
  \def\a{3.15}
  \def\b{1.75}
  \def\e{0.48}
  \def\ai{2.67}
  \def\bi{1.27}

  \fill[blue!8] (0,0) ellipse ({\a} and {\b});
  \fill[red!10, even odd rule] (0,0) ellipse ({\a} and {\b}) (0,0) ellipse ({\ai} and {\bi});
  \fill[green!14] (0,0) ellipse ({\ai} and {\bi});
  \draw[blue!70!black, thick] (0,0) ellipse ({\a} and {\b});
  \draw[green!55!black, thick, dashed] (0,0) ellipse ({\ai} and {\bi});

  \draw[gray!60, line width=0.7pt] (-1.35,0.35) -- (-0.75,0.65) -- (-0.05,0.55) -- (0.55,0.30) -- (1.05,-0.15);
  \draw[gray!60, line width=0.7pt] (-0.75,0.65) -- (-0.60,-0.20) -- (-0.05,0.55);
  \draw[gray!60, line width=0.7pt] (-0.60,-0.20) -- (0.20,-0.42) -- (0.55,0.30);

  \foreach \p in {(-1.35,0.35),(-0.75,0.65),(-0.05,0.55),(0.55,0.30),(1.05,-0.15),(-0.60,-0.20),(0.20,-0.42)} {
    \fill[green!55!black] \p circle (2.1pt);
  }

  \foreach \p in {(-2.70,0.12),(-2.25,0.95),(-1.15,1.43),(1.72,1.05),(2.72,-0.20),(1.60,-1.18),(-1.95,-1.05)} {
    \draw[red!75!black, fill=red!18, thick] \p circle (2.3pt);
  }

  \draw[<->, very thick] (0,-\b) -- (0,-\bi);
  \node[right] at (0.05,-1.51) {$\varepsilon_\tau$};
  \draw[densely dotted] (-0.22,-\b) -- (0.22,-\b);
  \draw[densely dotted] (-0.22,-\bi) -- (0.22,-\bi);

  \node[blue!70!black] at (-2.35,1.95) {$S_{j,\lambda}$};
  \node[green!50!black] at (0.15,1.45) {$S_{j,\lambda}^{-\varepsilon_\tau}$};
  \node[red!75!black] at (2.65,-1.50) {$\partial_{\varepsilon_\tau}S_{j,\lambda}$};
  \node[align=left] at (-0.05,-2.45) {
    green dots: samples covered by \(G_j(\tau)\) \\
    red dots: samples that may be missed by \(G_j(\tau)\)
  };
\end{tikzpicture}
\caption{Levelset view of KBC coverage. The blue region is the density
superlevel set \(S_{j,\lambda}\). The dashed green region is the eroded
interior \(S_{j,\lambda}^{-\varepsilon_\tau}\), where
\(\varepsilon_\tau=\sigma\sqrt{2\log(1/\tau)}\) is the graph radius induced by
the KBC threshold \(\tau\). Under the valid level set regime, KBC covers the
interior samples in its selected core \(G_j(\tau)\); samples not covered by the
core can only lie in the boundary band \(\partial_{\varepsilon_\tau}S_{j,\lambda}\).}
\label{fig:kbc-boundary-band}
\end{figure}
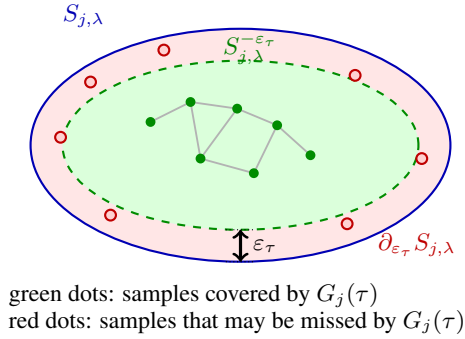

\begin{definition}[Interior and boundary band]
For \(\varepsilon>0\), define the \(\varepsilon\)-interior of the superlevel
set \(S_{j,\lambda}\) by
\[
S_{j,\lambda}^{-\varepsilon}
=
\{x\in S_{j,\lambda}: d(x,S_{j,\lambda}^c)>\varepsilon\}.
\]
The corresponding inner boundary band is
\[
\partial_\varepsilon S_{j,\lambda}
=
S_{j,\lambda}\setminus S_{j,\lambda}^{-\varepsilon}
=
\{x\in S_{j,\lambda}: d(x,S_{j,\lambda}^c)\le \varepsilon\},
\]
\end{definition}
where $\mathrm{dist}(x,S_{j,\lambda}^c)$ is the smallest distance between $x$ to the points in $S_{j,\lambda}^c$.

\begin{assumption}[Valid KBC level set regime]
\label{ass:valid-kbc-level-set-regime}
For the threshold \(\tau\), let
\(\varepsilon_\tau=\sigma\sqrt{2\log(1/\tau)}\). Assume that KBC selects a
nonempty component \(G_j(\tau)\subseteq U_j\) satisfying
\[
U_j\cap S_{j,\lambda}^{-\varepsilon_\tau}\subseteq G_j(\tau).
\]
Equivalently, all dense region samples at distance more than
\(\varepsilon_\tau\) from the level set boundary are covered by the selected
KBC core; missed dense region samples may occur only in the
\(\varepsilon_\tau\)-boundary band.
\end{assumption}

\begin{remark}
Assumption~\ref{ass:valid-kbc-level-set-regime} is the formal point where the
KBC threshold enters the analysis. It is expected to hold when
\(\varepsilon_\tau\) is large enough to connect the empirical samples inside
one density component but small enough not to merge distinct density
components.
\end{remark}

\subsubsection{KBC Core Bias}

\begin{lemma}[Coverage decomposition]
\label{lem:kbc-coverage-decomposition}
Let \(G_j(\tau)\subseteq U_j\), let \(R_j(\tau)=U_j\setminus G_j(\tau)\), and
define
\[
\alpha_j(\tau)=\frac{|R_j(\tau)|}{|U_j|}
=1-\frac{|G_j(\tau)|}{s_j}.
\]
If \(R_j(\tau)\neq\emptyset\), then
\[
\left\|\hat\mu_{U_j}-\hat\mu_{G_j(\tau)}\right\|_{\mathcal H}
=
\alpha_j(\tau)
\left\|\hat\mu_{R_j(\tau)}-\hat\mu_{G_j(\tau)}\right\|_{\mathcal H}.
\]
If \(R_j(\tau)=\emptyset\), both sides are zero.
\end{lemma}

\begin{proof}
Write \(m_j=|G_j(\tau)|\), \(r_j=|R_j(\tau)|\), and \(s_j=m_j+r_j\). If
\(r_j=0\), then \(U_j=G_j(\tau)\), and the result is immediate. Otherwise,
since \(U_j=G_j(\tau)\cup R_j(\tau)\) is a disjoint union,
\[
\hat\mu_{U_j}
=
\frac{m_j}{s_j}\hat\mu_{G_j(\tau)}
+
\frac{r_j}{s_j}\hat\mu_{R_j(\tau)}.
\]
Subtracting \(\hat\mu_{G_j(\tau)}\) gives
\[
\hat\mu_{U_j}-\hat\mu_{G_j(\tau)}
=
\frac{r_j}{s_j}
\left(\hat\mu_{R_j(\tau)}-\hat\mu_{G_j(\tau)}\right),
\]
and taking norms yields the claim.
\end{proof}

\begin{lemma}[Boundary band control of coverage loss]
\label{lem:kbc-boundary-coverage}
Under Assumption~\ref{ass:valid-kbc-level-set-regime},
\[
\alpha_j(\tau)
\le
\frac{1}{s_j}
\sum_{i=1}^{s_j}
\mathbf 1\{q_i\in \partial_{\varepsilon_\tau}S_{j,\lambda}\}.
\]
\end{lemma}

\begin{proof}
Assumption~\ref{ass:valid-kbc-level-set-regime} implies that any dense region
sample missed by \(G_j(\tau)\) cannot belong to the eroded interior
\(S_{j,\lambda}^{-\varepsilon_\tau}\). Since all samples in \(U_j\) lie in
\(S_{j,\lambda}\), every missed sample must lie in
\(\partial_{\varepsilon_\tau}S_{j,\lambda}\). Therefore
\[
R_j(\tau)\subseteq U_j\cap \partial_{\varepsilon_\tau}S_{j,\lambda}.
\]
Dividing cardinalities by \(s_j\) proves the result.
\end{proof}

\begin{theorem}[KBC Core Bias]
\label{thm:kbc-level-set-core-bias}
Fix \(0<\delta<1\) and a threshold \(\tau\in(0,1)\), and let
\[
p_j(\tau)=\mathcal{Q}_j(\partial_{\varepsilon_\tau}S_{j,\lambda}),
\qquad
\varepsilon_\tau=\sigma\sqrt{2\log(1/\tau)}.
\]
With probability at least \(1-\delta\) over
\(U_j=\{q_i\}_{i=1}^{s_j}\sim \mathcal{Q}_j^{s_j}\), the following implication holds:
if Assumption~\ref{ass:valid-kbc-level-set-regime} holds for the sampled
\(U_j\), then
\[
\left\|\hat\mu_{U_j}-\hat\mu_{G_j(\tau)}\right\|_{\mathcal H}
\le
2p_j(\tau)
+
2\sqrt{\frac{\log(1/\delta)}{2s_j}}.
\]
\end{theorem}

\begin{proof}
We separate the probabilistic part from the deterministic KBC coverage part.
The only random quantity that needs concentration is the empirical fraction
of dense region samples falling inside the boundary band.

\emph{Step 1: concentration of the empirical boundary band mass.}
For each dense region sample \(q_i\), define
\[
Z_i=\mathbf 1\{q_i\in \partial_{\varepsilon_\tau}S_{j,\lambda}\}.
\]
Since \(q_1,\ldots,q_{s_j}\) are i.i.d. samples from \(\mathcal{Q}_j\), the variables
\(Z_1,\ldots,Z_{s_j}\) are i.i.d. Bernoulli random variables. Their common
mean is the population boundary band mass:
\[
\mathbb E[Z_i]
=
\mathcal{Q}_j(\partial_{\varepsilon_\tau}S_{j,\lambda})
=
p_j(\tau).
\]
Hoeffding's inequality therefore gives, for every \(t>0\),
\[
\mathbb P\!\left(
\frac{1}{s_j}\sum_{i=1}^{s_j}Z_i-p_j(\tau)>t
\right)
\le
\exp(-2s_jt^2).
\]
Taking \(t=\sqrt{\log(1/\delta)/(2s_j)}\), we obtain the event
\(\mathcal E_\delta\):
\[
\frac{1}{s_j}\sum_{i=1}^{s_j}
\mathbf 1\{q_i\in \partial_{\varepsilon_\tau}S_{j,\lambda}\}
\le
p_j(\tau)+\sqrt{\frac{\log(1/\delta)}{2s_j}}.
\]
The probability of the complement of this event is at most
\[
\exp\!\left(
-2s_j\cdot \frac{\log(1/\delta)}{2s_j}
\right)
=
\exp(-\log(1/\delta))
=
\delta.
\]
Therefore \(\mathbb P(\mathcal E_\delta)\ge 1-\delta\). The rest of the proof
is deterministic on \(\mathcal E_\delta\) and under
Assumption~\ref{ass:valid-kbc-level-set-regime}.

\emph{Step 2: reduce core bias to coverage loss.}
If \(R_j(\tau)=\emptyset\), then
\(\hat\mu_{U_j}=\hat\mu_{G_j(\tau)}\), and the claimed bound is trivial.
Assume \(R_j(\tau)\neq\emptyset\). 
By Lemma~\ref{lem:kbc-coverage-decomposition},
\[
\left\|\hat\mu_{U_j}-\hat\mu_{G_j(\tau)}\right\|_{\mathcal H}
=
\alpha_j(\tau)
\left\|\hat\mu_{R_j(\tau)}-\hat\mu_{G_j(\tau)}\right\|_{\mathcal H}.
\]
For the Gaussian kernel, \(k_\sigma(x,x)=1\), so
\(\|\phi(x)\|_{\mathcal H}=1\). Hence every empirical mean embedding has
RKHS norm at most one, and
\[
\left\|\hat\mu_{R_j(\tau)}-\hat\mu_{G_j(\tau)}\right\|_{\mathcal H}
\le 2.
\]
Thus
\[
\left\|\hat\mu_{U_j}-\hat\mu_{G_j(\tau)}\right\|_{\mathcal H}
\le 2\alpha_j(\tau).
\]

\emph{Step 3: control the coverage loss by the boundary band.}
We now use the KBC level-set assumption. By
Assumption~\ref{ass:valid-kbc-level-set-regime}, every sample in the eroded
interior \(U_j\cap S_{j,\lambda}^{-\varepsilon_\tau}\) is included in the
selected core \(G_j(\tau)\). Hence any sample omitted by the core must lie
outside this eroded interior. Since all samples in \(U_j\) are drawn from
\(\mathcal{Q}_j=\mathcal{P}_j(\cdot\mid S_{j,\lambda})\), they lie in \(S_{j,\lambda}\). Therefore
the omitted set satisfies
\[
R_j(\tau)
=
U_j\setminus G_j(\tau)
\subseteq
U_j\cap
\bigl(S_{j,\lambda}\setminus S_{j,\lambda}^{-\varepsilon_\tau}\bigr)
=
U_j\cap \partial_{\varepsilon_\tau}S_{j,\lambda}.
\]
Taking cardinalities and dividing by \(s_j=|U_j|\) gives
\[
\alpha_j(\tau)
=
\frac{|R_j(\tau)|}{s_j}
\le
\frac{|U_j\cap \partial_{\varepsilon_\tau}S_{j,\lambda}|}{s_j}
=
\frac{1}{s_j}\sum_{i=1}^{s_j}Z_i.
\]
On \(\mathcal E_\delta\), the empirical average of the \(Z_i\)'s is at most
\[
\frac{1}{s_j}\sum_{i=1}^{s_j}Z_i
\le
p_j(\tau)+\sqrt{\frac{\log(1/\delta)}{2s_j}}.
\]
Combining this with the previous display yields
\[
\alpha_j(\tau)
\le
p_j(\tau)+\sqrt{\frac{\log(1/\delta)}{2s_j}}.
\]
Finally, Step 2 gives
\[
\left\|\hat\mu_{U_j}-\hat\mu_{G_j(\tau)}\right\|_{\mathcal H}
\le
2\alpha_j(\tau)
\le
2p_j(\tau)
+
2\sqrt{\frac{\log(1/\delta)}{2s_j}}.
\]
Thus the claimed implication holds on \(\mathcal E_\delta\). Since
\(\mathbb P(\mathcal E_\delta)\ge 1-\delta\), the theorem follows.
\end{proof}

\begin{assumption}[Linear boundary band mass]
\label{ass:kbc-linear-boundary}
There exists a constant \(C_j>0\) and a radius \(\varepsilon_0>0\) such that,
for every \(0<\varepsilon\le \varepsilon_0\),
\[
\mathcal{Q}_j(\partial_\varepsilon S_{j,\lambda})\le C_j\varepsilon.
\]
\end{assumption}

\begin{corollary}[Explicit dependence on \(\tau\)]
\label{cor:kbc-tau-explicit}
Fix \(0<\delta<1\) and \(\tau\in(0,1)\). Suppose
Assumption~\ref{ass:kbc-linear-boundary} holds and
\(\varepsilon_\tau=\sigma\sqrt{2\log(1/\tau)}\le \varepsilon_0\). Then, with
probability at least \(1-\delta\), the following implication holds: if
Assumption~\ref{ass:valid-kbc-level-set-regime} holds for the sampled \(U_j\),
then
\[
\left\|\hat\mu_{U_j}-\hat\mu_{G_j(\tau)}\right\|_{\mathcal H}
\le
2C_j\sigma\sqrt{2\log(1/\tau)}
+
2\sqrt{\frac{\log(1/\delta)}{2s_j}}.
\]
\end{corollary}

\begin{proof}
Apply Theorem~\ref{thm:kbc-level-set-core-bias} and use
Assumption~\ref{ass:kbc-linear-boundary} with
\(\varepsilon=\varepsilon_\tau\):
\[
p_j(\tau)
=\mathcal{Q}_j(\partial_{\varepsilon_\tau}S_{j,\lambda})
\le C_j\varepsilon_\tau
=C_j\sigma\sqrt{2\log(1/\tau)}.
\]
\end{proof}

\subsubsection{Relation to the Generic Coupling Bound}

The generic core bias Lemma \ref{lemma: cbias} in the main text gives $\frac{r}{\sigma}+B$ whenever the dense samples can be coupled to the selected core with Euclidean
radius \(r\) and imbalance \(B\). The level set certificate above is
complementary rather than universally tighter: it is KBC-specific because
\(\tau\) determines the graph radius \(\varepsilon_\tau\), which determines
the selected empirical density component and the boundary band mass
\(\mathcal{Q}_j(\partial_{\varepsilon_\tau}S_{j,\lambda})\).

When both assumptions are available, one can use the combined certificate
\[
\|\hat\mu_{U_j}-\hat\mu_{G_j(\tau)}\|_{\mathcal H}
\le
\min\!\left\{
\frac{r_j}{\sigma}+B_j,\;
2p_j(\tau)
+
2\sqrt{\frac{\log(1/\delta)}{2s_j}}
\right\}.
\]

\begin{remark}[Analysis of \(\tau\)]
For the Gaussian kernel, choosing \(\tau\) is equivalent to choosing the graph
radius \(\varepsilon_\tau=\sigma\sqrt{2\log(1/\tau)}\). A large \(\tau\)
gives a small radius and may fragment one dense component; a small \(\tau\)
gives a large radius and may merge different components. The useful regime is
the middle one: the graph connects samples inside each density superlevel set
but does not connect across low density gaps. In this regime, KBC can miss only
boundary band samples, so the core-bias term is controlled by the boundary band
mass \(p_j(\tau)=\mathcal{Q}_j(\partial_{\varepsilon_\tau}S_{j,\lambda})\) plus the
finite sample term \(O(s_j^{-1/2})\). Thus the theorem is a level set
certificate for appropriate \(\tau\), not an unconditional guarantee for every
threshold.
\end{remark}


\subsection{Limitations and Future Work}
\label{appendix: future-work-and-limitation}
Our theoretical analyses are conducted with Gaussian kernel. In practice, CaD clustering \citep{ting2022point,ting2026achieve,zhang2025kernel2,zhu2023kernel} is more effective with adaptive
kernels such as Isolation Kernel \citep{ting2018isolation}. The current theory explains the CaD clustering via Gaussian kernel only rather than providing a
complete guarantee for all practical kernel choices. Extending the analysis to
adaptive kernels is an important direction for future work.

The geometric coverage assumptions required for the core bias bound (Lemma \ref{lemma: cbias} and Assumption \ref{ass:valid-kbc-level-set-regime}) are highly sensitive to the curse of dimensionality when using a standard Euclidean-based Gaussian kernel. In high-dimensional spaces, satisfying these conditions without an exponentially growing sample size is practically prohibitive. This limitation is one reason why practical CaD clustering methods often use Isolation Kernel rather than Gaussian kernel: A recent work \citep{ting2024possible} shows that Isolation Kernel is the only measure that has the distinguishability property in high-dimensional space---enabling meaningful exact nearest neighbor search in high dimensional space, as opposed to other measures that produce meaningless nearest neighbors \citep{NN-meaningful-1999}.

An important future direction is to explore the theoretical limits of CaD clustering. Existing impossibility results for clustering \citep{kleinberg2002impossibility} are formulated
under a set-oriented view, where a cluster is treated as a set of similar points satisfying
certain abstract axioms. This set-oriented view and its resultant impossibility theorem are a mismatch to the CaD clustering, where
each cluster is treated as a set of i.i.d. points sampled from an unknown distribution. The empirical evidence in literature \citep{ting2022point,zhang2025kernel,zhang2025kernel2,zhu2023kernel,ting2026achieve} and the theory established in this paper point to a potential Possibility Theorem of Clustering, should the premise is cluster-as-distribution. 

\section{Experiment Settings}
\label{appendix: experiment setting}

All experiments reported in this paper were run on a Windows 10 machine
(x64) equipped with an Intel Core i7-1360P processor and 32~GB RAM. The synthetic experiments
reported here are based on CPU; the integrated GPU on the machine is Intel
Iris Xe Graphics. All the codes are available in \url{https://github.com/IsolationKernel/Codes}.


All synthetic experiments use the Gaussian kernel
$k_\sigma(x,y)=\exp(-\|x-y\|^2/(2\sigma^2))$ and set the number of clusters to the ground-truth value. We evaluate sample sizes
$n \in \{50, 100, 200, 500, 1000, 2000, 5000, 10000\}$.

For the controlled synthetic diagnostics in the main text, KBC
hyperparameters are selected by an \emph{oracle} protocol. We search over
the finite grid of Gaussian bandwidths and graph thresholds,
\[
\sigma \in \{0.3, 0.5, 1.0, 2.0, 5.0\},
\qquad
\tau \in \{0.1, 0.2, \ldots, 0.9\},
\]
and select, for each dataset, the pair $(\sigma,\tau)$ that maximizes NMI.
These oracle-tuned settings are then reused in the error decomposition plots.
Accordingly, the synthetic results should be interpreted as controlled
validation of the theory under favorable hyperparameter choices, rather than
as a claim about fully unsupervised model selection.

The evaluation metric is normalized mutual information (NMI). Ground truth
labels are used only for reporting NMI and for oracle selection of
$(\sigma,\tau)$ on synthetic data; they are not used by KBC during core
extraction or one-pass assignment. In the decomposition experiments,
truncation bias, estimation bias, core bias, and total error are computed via
their empirical RKHS / MMD estimators using the same Gaussian kernel, and the
theoretical upper bounds are evaluated with confidence parameter
$\delta=0.05$.

For the appendix sensitivity studies, Stage~1 is run on a random $70\%$
subsample, repeated over $10$ resamples, with \texttt{post\_process=False}
to isolate the effect of the initial threshold $\tau$. In the one-pass versus
iterative comparison, KBC is deterministic once $(\sigma,\tau)$ is fixed,
whereas Kernel $k$-means is run with up to $10$ random restarts and $100$
iterations, and we report its best NMI under the same Gaussian kernel family.


\subsection{Dataset Descriptions}
\label{app:datasets}

All five datasets are 2D and consist of $n$ points drawn i.i.d. from the corresponding distribution.

\begin{description}[leftmargin=1.5em, labelindent=0pt]

  \item[Two-Moons.]
    Two interleaved half-moon arcs, each containing roughly half the points.
    The clusters are non-convex and non-linearly separable, making this a
    canonical benchmark for kernel-based methods.
    Noise is added by sampling each point from a thin Gaussian tube centred on
    the arc.

  \item[Concentric-Rings.]
    Two nested annular clusters of similar width and mass.
    The inner ring is fully enclosed by the outer ring, so any method relying
    on convexity or linear separation will fail.
    Points are drawn uniformly on each annulus.

  \item[G-Strip.]
    One elongated cluster (the ``strip'') and one compact
    circular cluster.

  \item[EqSize-Gauss.]
    Two isotropic Gaussian clusters of equal size placed close together so
    that their supports overlap substantially.
    This tests the algorithm's ability to resolve clusters that are
    well separated in density but not in raw Euclidean distance.

  \item[Imbal-Gauss.]
    Three isotropic Gaussian clusters with a severe size imbalance
    (approximate ratio $8:2:1$).
    The minority cluster is easily overwhelmed by majority cluster points in
    density-based approaches.
\end{description}

For the additional qualitative comparison in
Figure~\ref{fig:clustering_vis}, we also include three benchmark
datasets taken directly from prior literature. Spiral and AC are the official synthetic
benchmarks used in Fig.~4 of the IDKC paper: Spiral consists of three
intertwined non-convex spiral arms, while AC is a highly non-convex benchmark
with multiple curved and adjacent components that is designed to stress the
geometry bias of centroid or variance driven methods. Two-Strips is the
harder official benchmark released with the KBC paper; it contains two
elongated strip-like clusters arranged in a way that makes their separation
challenging for standard Kernel $k$-means. On these three datasets, we keep the
original ground truth labels and report qualitative clustering assignments
together with NMI.

\section{Additional Empirical Results on Gaussian Distribution}
\label{appendix: additional experiments}

\subsection{Empirical Behavior of Truncation Bias under Different $\sigma$ and $\lambda$}
\label{appendix:tbias-empirical}

To complement Lemma \ref{lemma: tbias}, we study $\mathcal{P}=\mathcal N(0,I)$ in $\mathbb{R}^2$ with dynamic $\eta$ and compare
\begin{equation}
\text{Empirical Truncation Bias} = \|\mu_\mathcal{P}-\mu_\mathcal{Q}\|_{\mathcal{H}},\text{ Leading Term of Upper Bound} = \frac{\eta}{\sigma}\|\mathbb{E}[\mathcal{Q}]-\mathbb{E}[\mathcal{T}]\|.
\end{equation}
Figure \ref{fig: truncation_sigma_lambda_appendix} shows the expected asymptotic pattern for $\sigma \in \{0.001,0.01,0.1,1,10,100,1000\}$: the gap is larger at small $\sigma$ and becomes small at large $\sigma$. Across $\lambda$, the curves can be non-monotonic because both $\eta$ and $\|\mathbb{E}[\mathcal{Q}]-\mathbb{E}[\mathcal{T}]\|$ change with truncation level.

\begin{figure}[h]
    \centering
\includegraphics[width=0.95\linewidth]{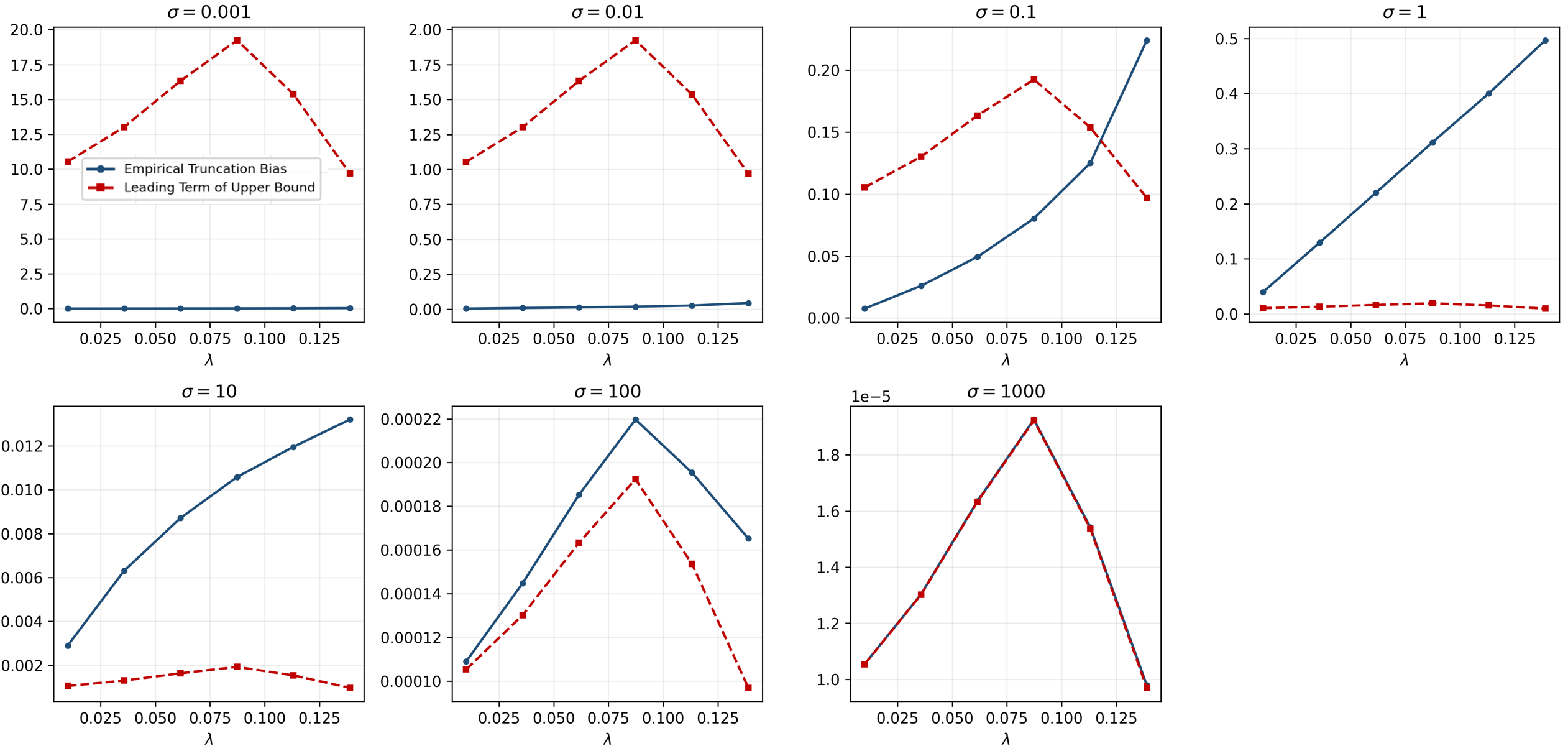}
    \caption{Empirical truncation bias and asymptotic leading term under different $\sigma$ values and density thresholds $\lambda$ (with dynamic $\eta$).}
    \label{fig: truncation_sigma_lambda_appendix}
\end{figure}

\subsection{Empirical Behavior of Estimation Bias under Different $\sigma$ and $\lambda$}
\label{appendix:ebias-empirical}

To complement Lemma \ref{lemma ebias}, we evaluate $\|\mu_\mathcal{Q}-\hat{\mu}_U\|_{\mathcal H}$ on the same 2D Gaussian setting ($\mathcal{P}=\mathcal N(0,I)$), varying $s$, with fixed $\delta=0.01$, multiple $\lambda$, and $\sigma\in\{10^{-3},10^{-2},10^{-1},1,10,100,1000\}$. Figure \ref{fig: estimation_s_lambda_sigma_appendix} uses a log-$y$ scale and shows that empirical estimation bias decreases with $s$ and follows the upper-bound trend $\frac{1+\sqrt{2\ln(1/\delta)}}{\sqrt{s}}$ up to sampling fluctuation.
Experimentally, for each $(\sigma,\lambda,s)$ setting, we first form $\mathcal{Q}$ by thresholding the same Gaussian sample pool, then repeatedly draw $s$ points $U$ from $\mathcal{Q}$ to build $\hat\mu_U$ and compute $\|\mu_\mathcal{Q}-\hat\mu_U\|_{\mathcal H}$. The solid curve is the trial mean, and the shaded band is mean $\pm$ one standard deviation across repeated trials, so the standard deviation quantifies Monte Carlo sampling variability under a fixed setting.

\begin{figure}[h]
    \centering
\includegraphics[width=0.98\linewidth]{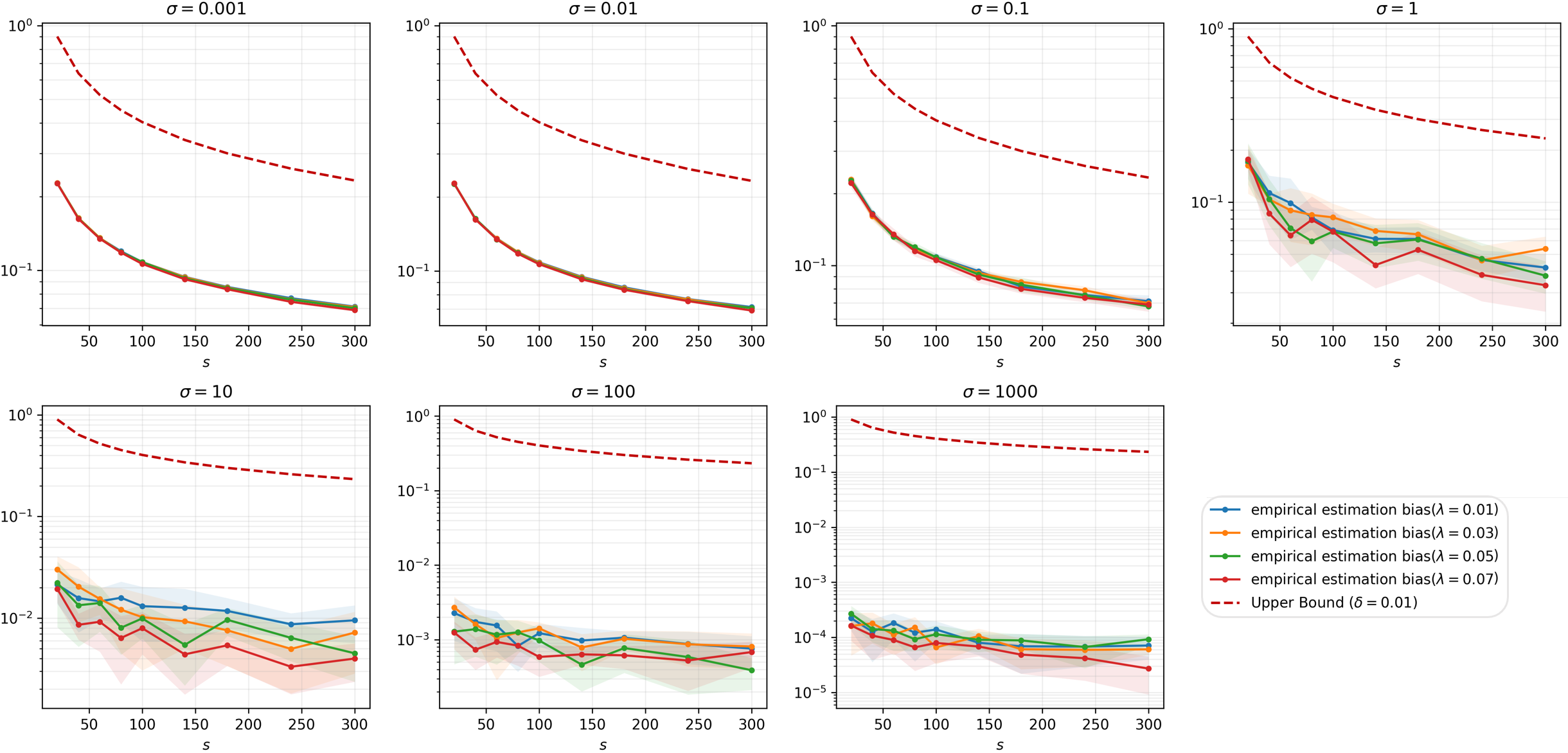}
    \caption{Empirical estimation bias versus $s$ under different $\sigma$ and $\lambda$, with fixed $\delta=0.01$, compared to the theoretical upper bound (log-$y$ scale).}
    \label{fig: estimation_s_lambda_sigma_appendix}
\end{figure}

\subsection{Empirical Behavior of Core Bias under Different $\tau$ (fixed $\sigma=1$)}
\label{appendix:cbias-empirical}

To complement Lemma \ref{lemma: cbias}, we evaluate
$\|\hat{\mu}_U-\hat{\mu}_G\|_{\mathcal H}$ on the same 2D Gaussian setting with fixed $\sigma=1$, varying $\tau$, multiple $\lambda$, and different $s$. Figure \ref{fig: core_bias_tau_appendix} compares empirical core bias with the bound $r/\sigma + B$ using a log-$y$ scale; the empirical curves remain below the bound across tested settings.
For each $(\tau,\lambda,s)$ setting, we repeatedly sample $s$ points $U$ from $\mathcal{Q}$, construct $G$ as the largest $\tau$-connected component, and compute both $\|\hat\mu_U-\hat\mu_G\|_{\mathcal H}$ and $r/\sigma+B$. The empirical curve reports the trial mean, while the shaded band is mean $\pm$ one standard deviation over trials; this standard deviation reflects randomness from repeated subset sampling (and the induced core extraction) under fixed parameters.

\begin{figure}[h]
    \centering
\includegraphics[width=0.98\linewidth]{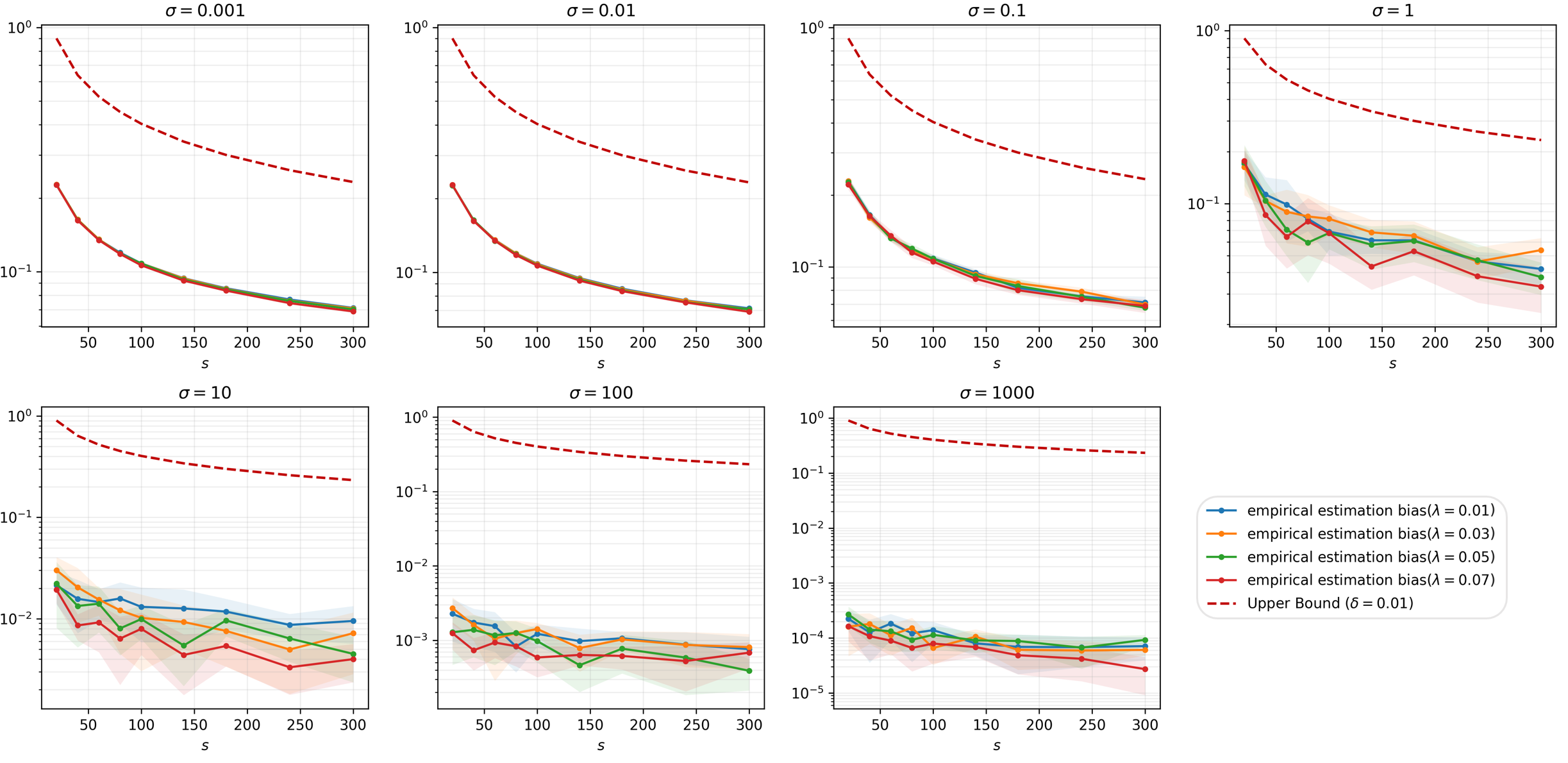}
    \caption{Core-bias validation with fixed $\sigma=1$ and varying $\tau$: empirical core bias versus the bound $r/\sigma + B$ under different $\lambda$ and $s$ (log-$y$ scale).}
    \label{fig: core_bias_tau_appendix}
\end{figure}

\subsection{Std Results corresponding to Figure \ref{fig:datasets_and_decomp}}
\label{appendix: stds}


For KBC, we report the detailed mean and std values corresponding to Figure \ref{fig:datasets_and_decomp} in Table \ref{tab:std_all}.
\begin{table}[ht]
\centering
\caption{KBC Mean $\pm$ std over 10 trials across all datasets and sample sizes.}
\label{tab:std_all}
\resizebox{\textwidth}{!}{%
\begin{tabular}{llcccccccc}
\toprule
Metric & Dataset & $n=50$ & $n=100$ & $n=200$ & $n=500$ & $n=1000$ & $n=2000$ & $n=5000$ & $n=10000$ \\
\midrule
\multirow{5}{*}{NMI} & Two-Moons & $0.203_{\pm 0.017}$ & $0.329_{\pm 0.000}$ & $0.889_{\pm 0.000}$ & $0.966_{\pm 0.000}$ & $0.926_{\pm 0.000}$ & $0.942_{\pm 0.002}$ & $0.958_{\pm 0.002}$ & $0.836_{\pm 0.279}$ \\
 & Concentric-Rings & $0.005_{\pm 0.005}$ & $0.003_{\pm 0.000}$ & $1.000_{\pm 0.000}$ & $1.000_{\pm 0.000}$ & $1.000_{\pm 0.000}$ & $1.000_{\pm 0.000}$ & $0.900_{\pm 0.300}$ & $1.000_{\pm 0.000}$ \\
 & G-Strip & $0.610_{\pm 0.119}$ & $1.000_{\pm 0.000}$ & $1.000_{\pm 0.000}$ & $1.000_{\pm 0.000}$ & $1.000_{\pm 0.000}$ & $1.000_{\pm 0.000}$ & $1.000_{\pm 0.000}$ & $1.000_{\pm 0.000}$ \\
 & EqSize-Gauss & $0.772_{\pm 0.144}$ & $0.728_{\pm 0.000}$ & $0.797_{\pm 0.000}$ & $0.569_{\pm 0.000}$ & $0.811_{\pm 0.000}$ & $0.754_{\pm 0.084}$ & $0.764_{\pm 0.045}$ & $0.761_{\pm 0.082}$ \\
 & Imbal-Gauss & $1.000_{\pm 0.000}$ & $0.724_{\pm 0.230}$ & $1.000_{\pm 0.000}$ & $1.000_{\pm 0.000}$ & $0.990_{\pm 0.000}$ & $0.987_{\pm 0.002}$ & $0.985_{\pm 0.003}$ & $0.988_{\pm 0.001}$ \\
\midrule
\multirow{5}{*}{Total Error} & Two-Moons & $0.796_{\pm 0.013}$ & $0.680_{\pm 0.004}$ & $0.307_{\pm 0.000}$ & $0.002_{\pm 0.000}$ & $0.001_{\pm 0.000}$ & $0.001_{\pm 0.000}$ & $0.000_{\pm 0.000}$ & $0.000_{\pm 0.000}$ \\
 & Concentric-Rings & $0.575_{\pm 0.000}$ & $0.284_{\pm 0.000}$ & $0.000_{\pm 0.000}$ & $0.000_{\pm 0.000}$ & $0.000_{\pm 0.000}$ & $0.000_{\pm 0.000}$ & $0.000_{\pm 0.000}$ & $0.000_{\pm 0.000}$ \\
 & G-Strip & $0.274_{\pm 0.006}$ & $0.126_{\pm 0.000}$ & $0.084_{\pm 0.000}$ & $0.002_{\pm 0.000}$ & $0.001_{\pm 0.000}$ & $0.000_{\pm 0.000}$ & $0.000_{\pm 0.000}$ & $0.000_{\pm 0.000}$ \\
 & EqSize-Gauss & $0.087_{\pm 0.000}$ & $0.638_{\pm 0.012}$ & $0.572_{\pm 0.010}$ & $0.462_{\pm 0.000}$ & $0.346_{\pm 0.000}$ & $0.261_{\pm 0.000}$ & $0.057_{\pm 0.000}$ & $0.022_{\pm 0.000}$ \\
 & Imbal-Gauss & $0.116_{\pm 0.000}$ & $0.054_{\pm 0.000}$ & $0.014_{\pm 0.000}$ & $0.007_{\pm 0.000}$ & $0.003_{\pm 0.000}$ & $0.001_{\pm 0.000}$ & $0.000_{\pm 0.000}$ & $0.000_{\pm 0.000}$ \\
\midrule
\multirow{5}{*}{Estimation Bias} & Two-Moons & $0.606_{\pm 0.136}$ & $0.245_{\pm 0.042}$ & $0.034_{\pm 0.000}$ & $0.000_{\pm 0.000}$ & $0.001_{\pm 0.000}$ & $0.001_{\pm 0.000}$ & $0.000_{\pm 0.000}$ & $0.000_{\pm 0.000}$ \\
 & Concentric-Rings & $0.378_{\pm 0.000}$ & $0.000_{\pm 0.000}$ & $0.000_{\pm 0.000}$ & $0.000_{\pm 0.000}$ & $0.000_{\pm 0.000}$ & $0.000_{\pm 0.000}$ & $0.000_{\pm 0.000}$ & $0.000_{\pm 0.000}$ \\
 & G-Strip & $0.056_{\pm 0.003}$ & $0.011_{\pm 0.000}$ & $0.000_{\pm 0.000}$ & $0.002_{\pm 0.000}$ & $0.001_{\pm 0.000}$ & $0.000_{\pm 0.000}$ & $0.000_{\pm 0.000}$ & $0.000_{\pm 0.000}$ \\
 & EqSize-Gauss & $0.087_{\pm 0.000}$ & $0.411_{\pm 0.104}$ & $0.267_{\pm 0.089}$ & $0.135_{\pm 0.000}$ & $0.104_{\pm 0.000}$ & $0.074_{\pm 0.000}$ & $0.024_{\pm 0.000}$ & $0.010_{\pm 0.000}$ \\
 & Imbal-Gauss & $0.024_{\pm 0.000}$ & $0.018_{\pm 0.000}$ & $0.007_{\pm 0.000}$ & $0.003_{\pm 0.000}$ & $0.001_{\pm 0.000}$ & $0.001_{\pm 0.000}$ & $0.000_{\pm 0.000}$ & $0.000_{\pm 0.000}$ \\
\midrule
\multirow{5}{*}{Core Bias} & Two-Moons & $0.796_{\pm 0.013}$ & $0.680_{\pm 0.004}$ & $0.307_{\pm 0.000}$ & $0.002_{\pm 0.000}$ & $0.001_{\pm 0.000}$ & $0.001_{\pm 0.000}$ & $0.000_{\pm 0.000}$ & $0.000_{\pm 0.000}$ \\
 & Concentric-Rings & $0.575_{\pm 0.000}$ & $0.284_{\pm 0.000}$ & $0.000_{\pm 0.000}$ & $0.000_{\pm 0.000}$ & $0.000_{\pm 0.000}$ & $0.000_{\pm 0.000}$ & $0.000_{\pm 0.000}$ & $0.000_{\pm 0.000}$ \\
 & G-Strip & $0.274_{\pm 0.006}$ & $0.126_{\pm 0.000}$ & $0.084_{\pm 0.000}$ & $0.002_{\pm 0.000}$ & $0.001_{\pm 0.000}$ & $0.000_{\pm 0.000}$ & $0.000_{\pm 0.000}$ & $0.000_{\pm 0.000}$ \\
 & EqSize-Gauss & $0.087_{\pm 0.000}$ & $0.638_{\pm 0.012}$ & $0.572_{\pm 0.010}$ & $0.462_{\pm 0.000}$ & $0.346_{\pm 0.000}$ & $0.261_{\pm 0.000}$ & $0.057_{\pm 0.000}$ & $0.022_{\pm 0.000}$ \\
 & Imbal-Gauss & $0.116_{\pm 0.000}$ & $0.054_{\pm 0.000}$ & $0.014_{\pm 0.000}$ & $0.007_{\pm 0.000}$ & $0.003_{\pm 0.000}$ & $0.001_{\pm 0.000}$ & $0.000_{\pm 0.000}$ & $0.000_{\pm 0.000}$ \\
\midrule
\multirow{5}{*}{Truncation Bias} & Two-Moons & $0.606_{\pm 0.136}$ & $0.245_{\pm 0.042}$ & $0.034_{\pm 0.000}$ & $0.000_{\pm 0.000}$ & $0.001_{\pm 0.000}$ & $0.001_{\pm 0.000}$ & $0.000_{\pm 0.000}$ & $0.000_{\pm 0.000}$ \\
 & Concentric-Rings & $0.378_{\pm 0.000}$ & $0.000_{\pm 0.000}$ & $0.000_{\pm 0.000}$ & $0.000_{\pm 0.000}$ & $0.000_{\pm 0.000}$ & $0.000_{\pm 0.000}$ & $0.000_{\pm 0.000}$ & $0.000_{\pm 0.000}$ \\
 & G-Strip & $0.056_{\pm 0.003}$ & $0.011_{\pm 0.000}$ & $0.000_{\pm 0.000}$ & $0.002_{\pm 0.000}$ & $0.001_{\pm 0.000}$ & $0.000_{\pm 0.000}$ & $0.000_{\pm 0.000}$ & $0.000_{\pm 0.000}$ \\
 & EqSize-Gauss & $0.087_{\pm 0.000}$ & $0.411_{\pm 0.104}$ & $0.267_{\pm 0.089}$ & $0.135_{\pm 0.000}$ & $0.104_{\pm 0.000}$ & $0.074_{\pm 0.000}$ & $0.024_{\pm 0.000}$ & $0.010_{\pm 0.000}$ \\
 & Imbal-Gauss & $0.024_{\pm 0.000}$ & $0.018_{\pm 0.000}$ & $0.007_{\pm 0.000}$ & $0.003_{\pm 0.000}$ & $0.001_{\pm 0.000}$ & $0.001_{\pm 0.000}$ & $0.000_{\pm 0.000}$ & $0.000_{\pm 0.000}$ \\
\bottomrule
\end{tabular}%
}
\end{table}

\subsection{Clustering Result Visualization}
\label{appendix:cvis}
We show the clustering result of KBC and the best clustering result of KKM in Figure \ref{fig:clustering_vis}. 

\begin{figure}[ht]
  \centering
\includegraphics[width=0.99\textwidth]{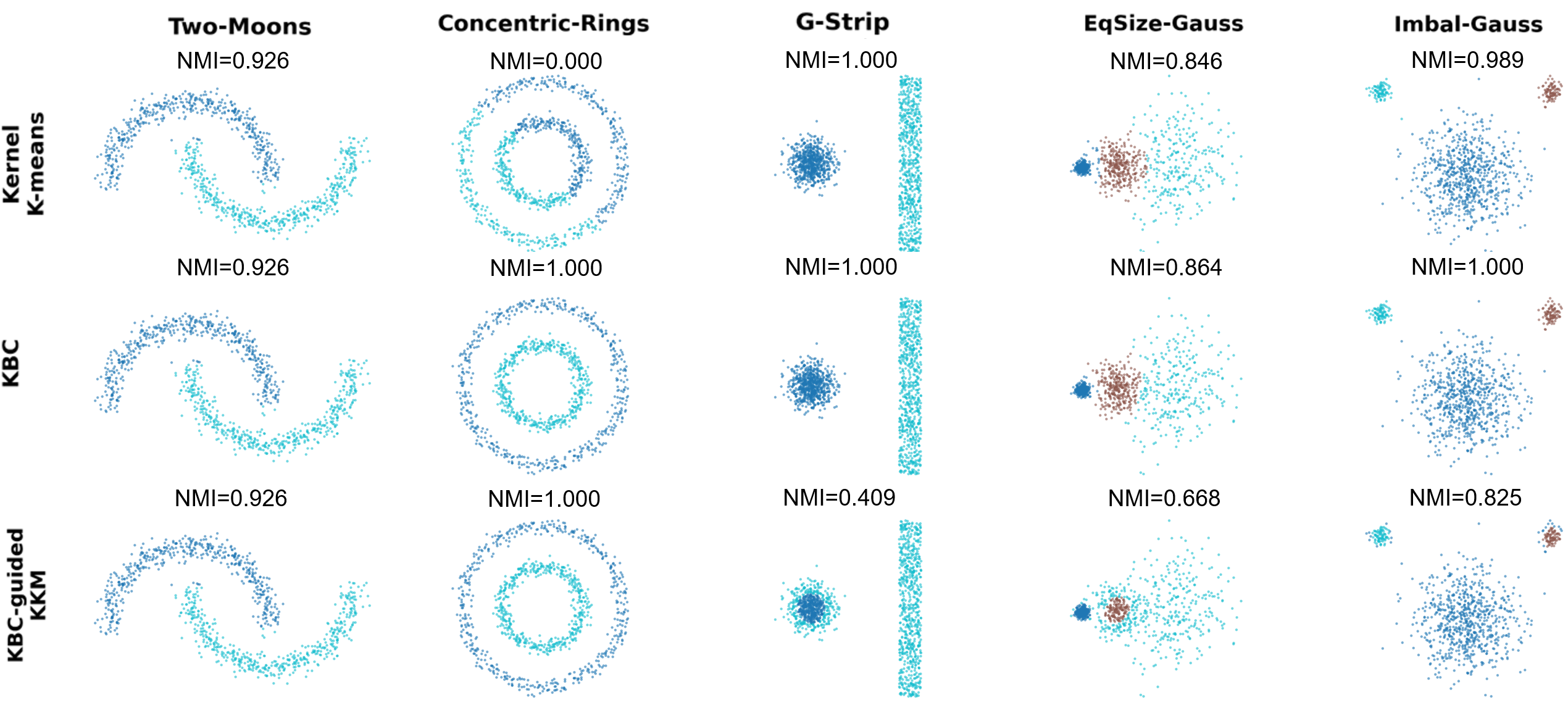}
\includegraphics[width=0.6\textwidth]{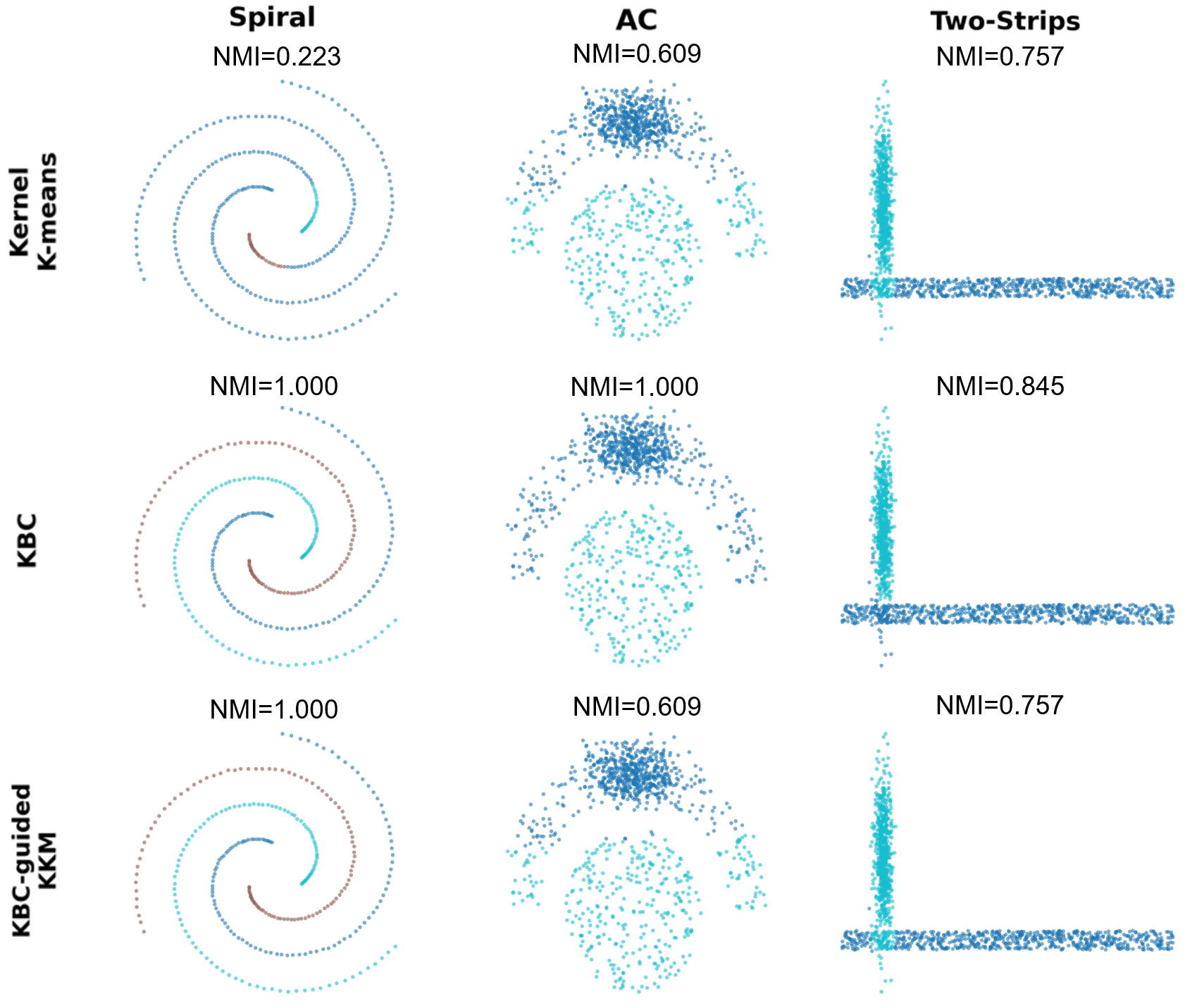}
\caption{Clustering Results.}
  \label{fig:clustering_vis}
\end{figure}

\subsection{Sensitivity of Core Bias}
\label{appendix:sensitivity core bias}
This subsection examines how the Stage 1 threshold $\tau$ in KBC affects core quality and, consequently, the final clustering performance of KBC. Across the five datasets, the experiments show that core bias is highly sensitive to $\tau$, and that the values of $\tau$ yielding lower core bias usually also produce higher or near-highest NMI. In particular, excessively small $\tau$ may make the threshold graph overly connected and prevent the extraction of $k$ valid cores, especially on non-convex datasets. Overall, these results confirm that the effectiveness of the assignment step depends critically on the quality of the extracted cores, making $\tau$ selection a central issue rather than a minor implementation detail.

To complement this sensitivity analysis, we further introduce a simple unsupervised heuristic for selecting $\tau$ directly from the connected components produced in the first step of KBC. The heuristic favors candidate thresholds that yield connected components with large coverage, strong internal cohesion, and balanced sizes, thereby reflecting the structural quality of the extracted cores without using label information. Empirically, the resulting $\tau$ values are well aligned with the best-performing choices on most datasets, indicating that the proposed criterion provides a practical and effective rule for threshold selection in KBC.

\subsubsection{Results on Synthetic Datasets}
\label{app_heuristic}
\begin{figure}[ht]
  \centering
  \includegraphics[width=\textwidth]{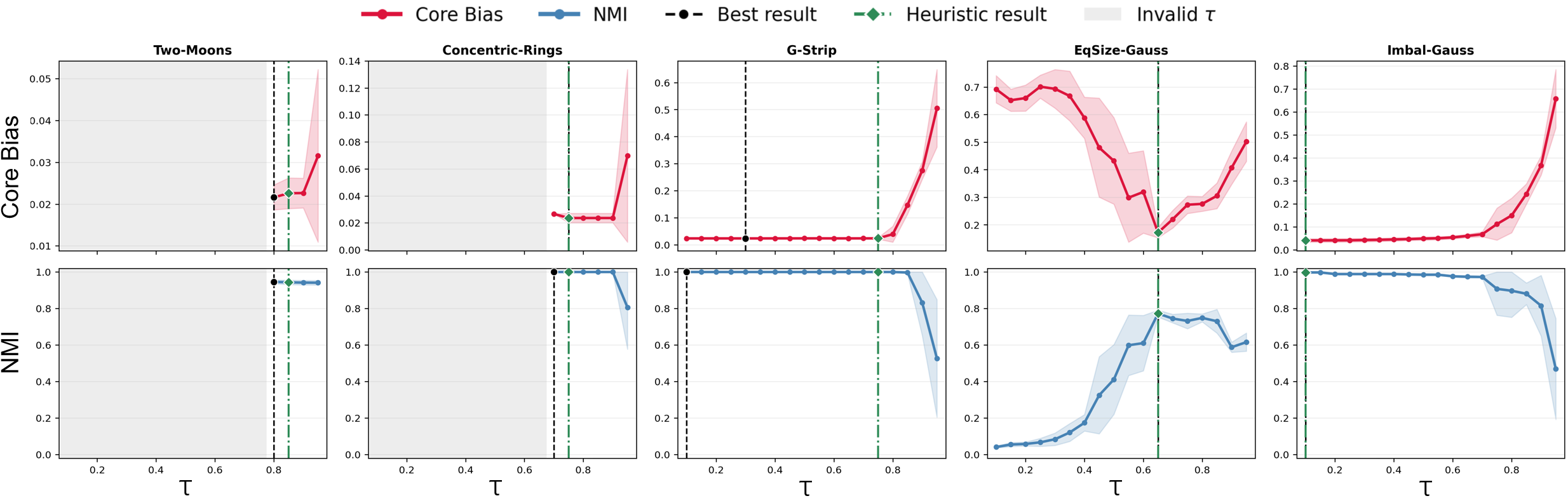}
  \caption{Sensitivity of one-pass KBC to the core extraction threshold $\tau$ across five datasets. For each dataset, we report the mean and standard deviation of core bias and final NMI over 10 subsamples. The gray area means that the invalid $\tau$ values (small $\tau$) makes the connected component contain all the points and hence core bias can not be computed. The sigma here is selected via common median heuristic.
  }
  \label{fig:core-bias-tau}
\end{figure}

To evaluate the importance of core quality, we study the sensitivity of KBC to the threshold parameter $\tau$ used in Stage 1 for core extraction. For each dataset and each value of $\tau$, we repeat the subsampling procedure 10 times using 70\% random subsamples, disable post-processing, and measure both the resulting core bias and the final one-pass NMI. The top row of Fig.~\ref{fig:core-bias-tau} reports the mean RKHS discrepancy between each true cluster and its matched extracted core, while the bottom row reports the corresponding clustering performance. 

Several consistent patterns emerge.

\begin{enumerate}
    \item First, the choice of $\tau$ has a direct and often substantial effect on core quality, confirming that core extraction is a determining factor in the overall method.
    \item Second, on most datasets, the range of $\tau$ values that minimizes core bias also yields the highest or near-highest NMI, showing that better cores lead to better one-pass assignments. 
    \item Third, the gray regions on non-convex datasets such as Two-Moons and Concentric-Rings indicate that small $\tau$ values can make the threshold graph overly connected, preventing the extraction of $k$ valid cores. 
\end{enumerate} Overall, these results reinforce the central message of the theory: \textbf{optimal assignment is meaningful only relative to the extracted cores, and therefore the quality of the cores is itself a crucial object of analysis}. The sensitivity of the core bias could be mitigated via specific kernel choices, as shown in Figure A3 in \citep{zhang2025kernel2}, which uses a more complicated Isolation Kernel.

\paragraph{A heuristic for choosing $\tau$ based on connected components.}
To select $\tau$ without using labels, we score each candidate $\tau$ by the quality of the top-$k$ connected components returned by the first step of KBC. Let $X_{\mathrm{sub}}$ be the subsample used in Step 1, and let $G_1(\tau),\dots,G_k(\tau)$ denote the $k$ largest connected components extracted from the kernel-threshold graph under parameter $\tau$. We define
\[
S_{\mathrm{heur}}(\tau)
:=
\operatorname{Cov}(\tau)\,
\operatorname{Coh}(\tau)\,
\sqrt{\operatorname{Bal}(\tau)},
\]
where
\[
\operatorname{Cov}(\tau)
:=
\frac{\sum_{j=1}^k |G_j(\tau)|}{|X_{\mathrm{sub}}|},
\qquad
\operatorname{Bal}(\tau)
:=
\frac{\min_{1 \le j \le k}|G_j(\tau)|}{\max_{1 \le j \le k}|G_j(\tau)|},
\]
and
\[
\operatorname{Coh}(\tau)
:=
\frac{1}{k}\sum_{j=1}^k
\frac{1}{|G_j(\tau)|}
\sum_{x \in G_j(\tau)} K\!\left(\delta(x), G_j(\tau)\right),
\]

where $K$ is distributional kernel, $\operatorname{Cov}(\tau)$ quantifies the coverage of the subsample by the extracted top-$k$ connected components, $\operatorname{Coh}(\tau)$ quantifies their average within-component cohesion in terms of the distribution kernel, and $\operatorname{Bal}(\tau)$ quantifies the size balance among these components. The square root is introduced to soften the effect of the balance term, so that it discourages severely imbalanced connected components while still allowing coverage and cohesion to remain the primary factors in selecting $\tau$.

 Equivalently,
\[
\operatorname{Coh}(\tau)
=
\frac{1}{k}\sum_{j=1}^k
\frac{1}{|G_j(\tau)|^2}
\sum_{x,g \in G_j(\tau)} k(x,g).
\]
The heuristic choice of $\tau$ is then
\[
\tau_{\mathrm{heur}}
\in
\arg\max_{\tau \in T} S_{\mathrm{heur}}(\tau),
\]
where $T$ is the candidate set of threshold parameters. Intuitively, this criterion prefers values of $\tau$ that produce connected components with large coverage, strong internal cohesion, and balanced component sizes.

Empirically, the proposed connected-component heuristic yields $\tau$ values that are well aligned with the best empirical choices. Overall, the heuristic offers an effective unsupervised rule for choosing $\tau$ in KBC, as shown in Figure \ref{fig:core-bias-tau}.





\subsubsection{Results on Real Datasets}
\label{appendix:heur}

\begin{figure}[h]
    \centering
\includegraphics[width=.95\textwidth]{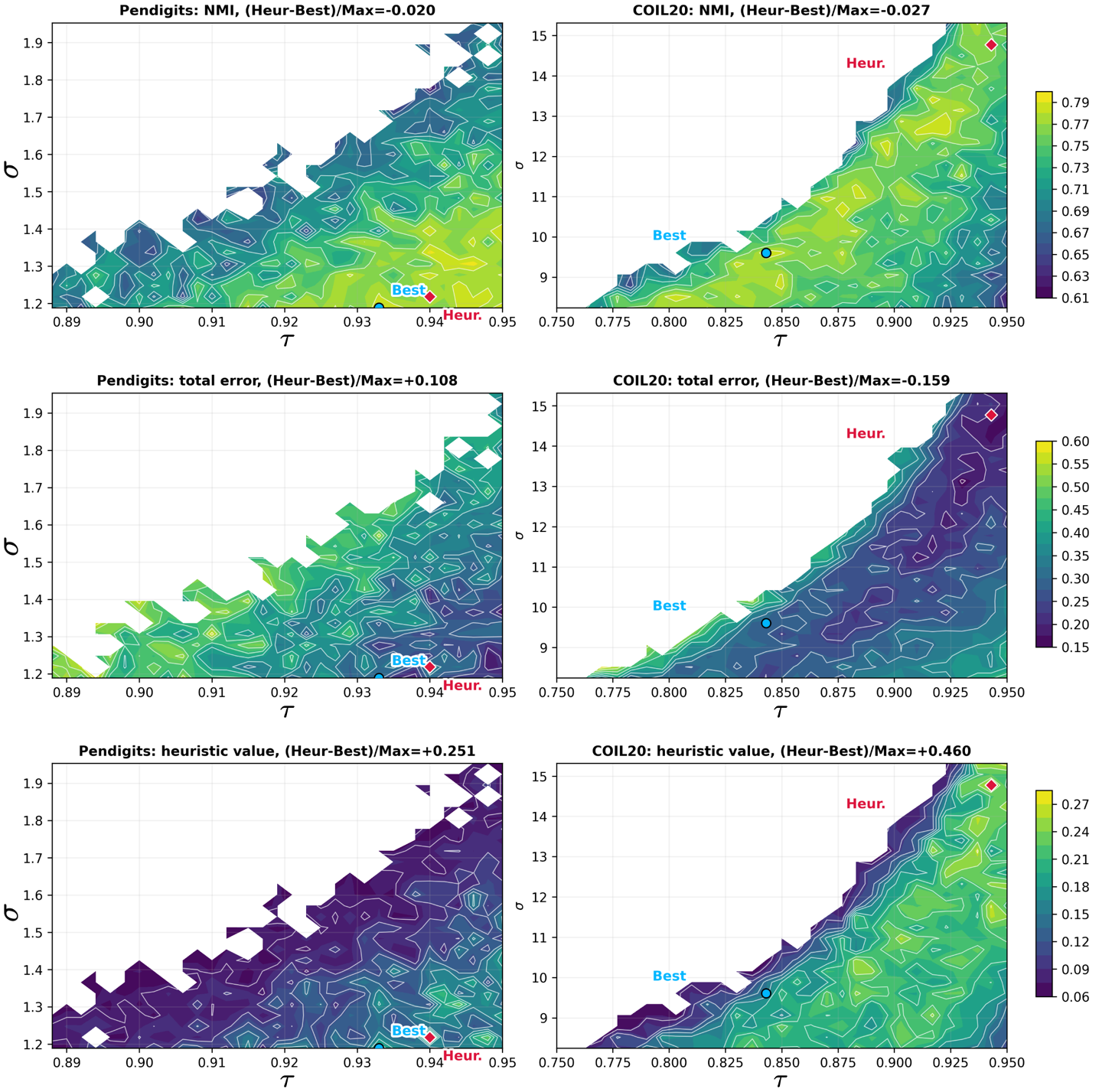}
    \caption{
    Refined feasible-region analysis of KBC on \textsc{Pendigits} and \textsc{COIL20} with MinMax normalization. 
    Each column shows the local $(\tau,\sigma)$ search region using three aligned heatmaps: NMI, total error, and heuristic value. 
    \textbf{Best} denotes the feasible grid point with the maximum NMI in the refined search region, while \textbf{Heur.} denotes the feasible grid point with the maximum heuristic value. 
    Hence, \textbf{Best} is selected strictly by clustering performance rather than by the heuristic criterion. 
    On both datasets, the heuristic optimum is close to the NMI-optimal point, showing that the heuristic can effectively localize a strong parameter region even when it does not exactly match the NMI maximizer.
    }
    \label{fig:kbc-pendigits-coil20-refined}
\end{figure}

Figure~\ref{fig:kbc-pendigits-coil20-refined} highlights the relationship between the heuristic-guided parameter choice and the true clustering optimum (w.r.t NMI) in the feasible $(\tau,\sigma)$ region. 
The main point is that \textbf{Best} is defined as the feasible parameter pair achieving the highest NMI over the grid, not the pair with the largest heuristic score. 
This distinction is crucial because it separates the role of the heuristic as a selection proxy from the actual evaluation criterion used to determine the best clustering result.



For \textsc{Pendigits}, the NMI optimal point is $(0.93,\,1.12)$ with NMI $0.71$, whereas the heuristic-selected point is $(0.94,\,1.22)$ with NMI $0.70$. 
For \textsc{COIL20}, the NMI optimal point is $(0.84,\,9.61)$ with NMI $0.7944$, while the heuristic-selected point is $(0.94,\,14.77)$ with NMI $0.77$. 
In both cases, the heuristic choice lies near the optimal region w.r.t NMI and achieves competitive performance, although it does not exactly coincide with the NMI maximizer.


Figure \ref{fig:kbc-pendigits-coil20-refined} examines whether the proposed heuristic can identify a useful hyperparameter region without directly optimizing NMI. On both Pendigits and COIL20, the heuristic-selected point lies close to the best feasible region under the ground-truth NMI grid. The selected pair does not exactly coincide with the NMI maximizer, but its performance loss is small: about 2.0\% on Pendigits and 2.7\% on COIL20 relative to the best observed NMI. This suggests that the heuristic captures the similar qualitative structure as the grid search.

The aligned heatmaps also show that low total error and high NMI tend to occur in nearby regions, supporting the link between core approximation quality and downstream assignment performance. However, the match is not perfect, which is expected on real datasets. Thus, Figure \ref{fig:kbc-pendigits-coil20-refined} should be read as practical evidence that the heuristic can localize good parameter choices.

\subsection{Empirical validation of another CaD clustering method: GDKC}
\label{app:gdkc}

\begin{figure}[htbp]
  \centering
  \includegraphics[width=\textwidth]{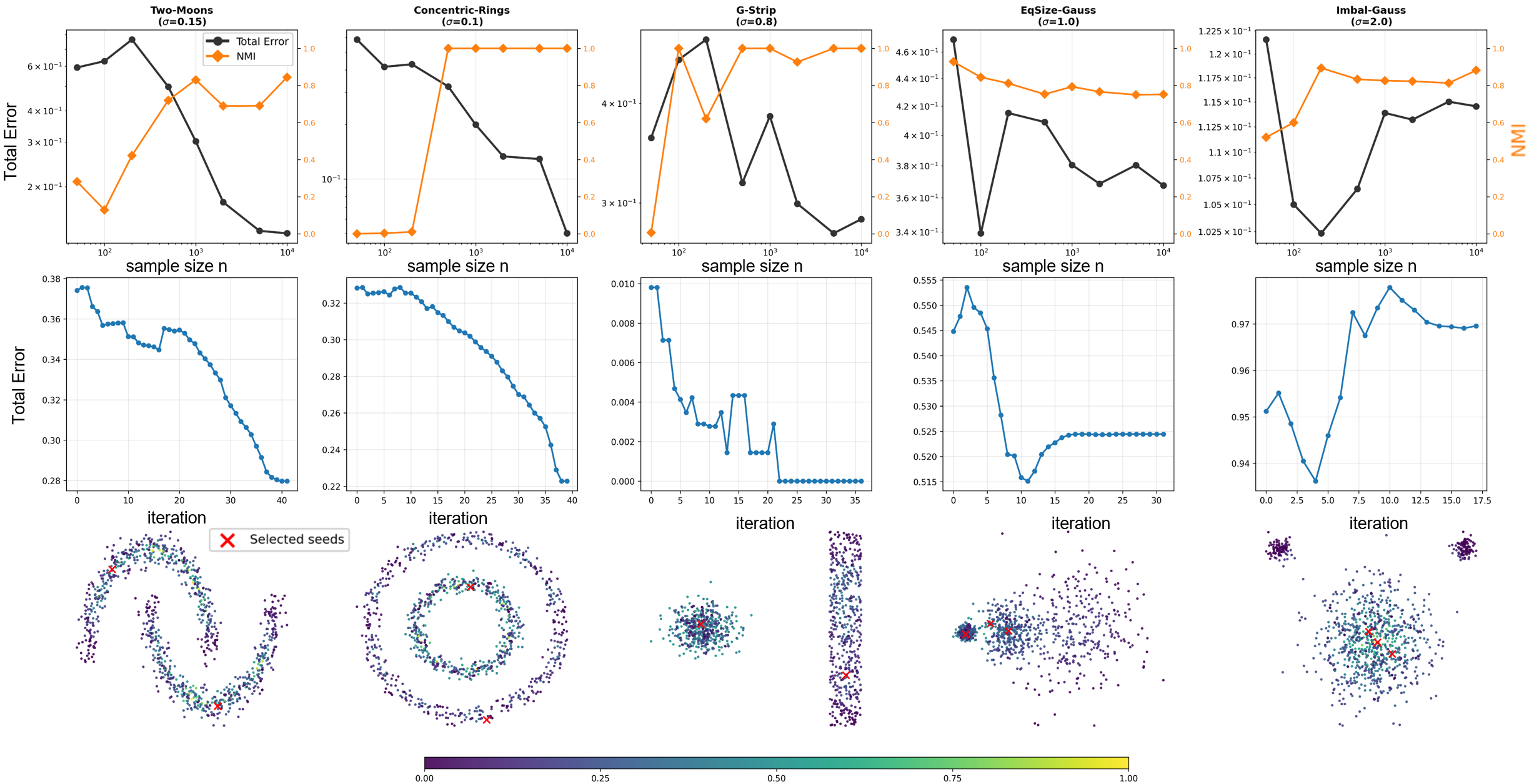}
  \caption{%
    \textbf{Results of GDKC on five synthetic datasets.}
    \emph{Top row}: Total Error 
    and NMI vs.\ dataset size $n$, averaged over 10 trials.
    \emph{Middle row}: Total Error across grow iterations at $n=1000$,
    showing how the $\tau$-decay growth phase reduces the total error. \emph{Bottom row}: seed selection score and selected seeds at $n=1000$. $\delta$ is selected using the heuristic method without the Cov term, because the core clusters are growing subsets rather than fixed ones in KBC. 
  }
  \label{fig:gdkc_combined}
\end{figure}

Figure~\ref{fig:gdkc_combined} reports the Total Error 
and NMI on five synthetic datasets.
For non-convex and strip clusters, total error decreases
monotonically with $n$ while NMI approaches $1.0$.
For Gaussian clusters the error converges to a positive constant, as the
density peak core is structurally a strict subset of the Gaussian component;
nevertheless NMI remains stable (${\approx}0.80$), confirming that
irreducible core bias does not impair clustering accuracy.
The grow traces (bottom row) show monotone error reduction for non-convex
data and a brief non-monotone transient for Gaussian data before settling
at the respective plateau.

For the two Gaussian datasets (EqSize-Gauss, Imbal-Gauss), the total error does not vanish as $n$ increases, because the seeds favor high density areas and the seeds are not  distributed in all clusters, resulting in a persistent distributional gap. The NMI is also lower due to bad cores. The NMI on imbal-Gauss is rather high because the largest cluster takes 8/11 of the mass. So the cluster results of the top two small Gaussian clusters have a smaller effect on the NMI.



\section{On the use of Large Language Models}
\label{appendix: LLM}

 Large Language Models (LLMs) were utilized in the polishing phase of this paper’s preparation. Specifically, LLMs were employed to optimize linguistic clarity, enhance stylistic coherence, and correct minor grammatical or syntactical inconsistencies. All core intellectual content, including conceptual frameworks, empirical observations, argumentative structure, and citation alignment, was developed, curated, and validated exclusively by the human authors.

 Portions of the code in this project were developed with the assistance of large language models (LLMs), including support for code scaffolding, debugging, and documentation refinement. All LLM-generated or LLM-modified content was manually reviewed, tested, and revised as necessary by the authors. The authors take full responsibility for the correctness and reliability of the final code and the resulting research outcomes.


\end{document}